\documentclass[twoside]{article}

\usepackage[accepted]{aistats2023}

\usepackage{color}
\usepackage{comment}
\usepackage{subfig}
\usepackage{lipsum}
\usepackage[hidelinks]{hyperref}
\usepackage[round]{natbib}
\usepackage{xr}
\usepackage{xargs}

\usepackage{booktabs}
\usepackage[usenames,dvipsnames]{xcolor}
\usepackage{latexsym}              % symbols
\usepackage{amsmath}               % great math stuff
\usepackage{amssymb}               % great math symbols
\usepackage{amsfonts}              % for blackboard bold, etc
\usepackage{amsthm}                % for theorems, http://tex.stackexchange.com/a/130655
\usepackage{accents}
\usepackage{mathtools}
\usepackage{multirow}
\usepackage{tikz}                  % to make drawing with TikZ
\usetikzlibrary{arrows,positioning,shapes}
\usepackage{pifont}% http://ctan.org/pkg/pifont
\usepackage[ruled,vlined]{algorithm2e}
\usepackage{enumitem}
\usetikzlibrary{matrix}

\usepackage[mathcal]{eucal}
\usepackage{breakcites}

\usepackage{cleveref}
\crefname{assumption}{Assumption}{Assumptions}
\crefname{equation}{Eq.}{Eqs.}
\crefname{figure}{Fig.}{Figs.}
\crefname{table}{Table}{Tables}
\crefname{section}{Sec.}{Secs.}
\crefname{algorithm}{Algorithm}{Algorithms}
\crefname{theorem}{Thm.}{Thms.}
\crefname{lemma}{Lemma}{Lemmas}
\crefname{proposition}{Prop.}{Propositions}
\crefname{corollary}{Cor.}{Cors.}
\crefname{example}{Example}{Examples}
\crefname{appendix}{Appendix}{Appendixes}
\crefname{remark}{Remark}{Remark}

\newcommand{\note}[1]{{\textbf{\color{red}#1}}}

\newcommand{\A}{\operatorname{A}}
\newcommand{\B}{\operatorname{B}}
\newcommand{\C}{\operatorname{C}}
\newcommand{\D}{\operatorname{D}}
\newcommand{\PP}{\operatorname{P}}
\newcommand{\Q}{\operatorname{Q}}

\newcommand{\density}[2]{\mu(#1|#2)}
\newcommand{\survival}[2]{S(#1|#2)}

\newtheorem{innercustomgeneric}{\customgenericname}
\providecommand{\customgenericname}{}
\newcommand{\newcustomtheorem}[2]{%
  \newenvironment{#1}[1]
  {%
   \renewcommand\customgenericname{#2}%
   \renewcommand\theinnercustomgeneric{##1}%
   \innercustomgeneric
  }
  {\endinnercustomgeneric}
}
\newcustomtheorem{customthm}{Theorem}
\newcustomtheorem{customlemma}{Lemma}
\newcustomtheorem{customprop}{Proposition}

\newcounter{remark}[section]

\newcommand{\Cind}{\operatorname{C}}

\newcommand{\calX}{\mathcal{X}}

\newcommand{\calZ}{\mathcal{Z}}

\newcommand{\calE}{\mathcal{E}}

\newcommand{\calC}{\mathcal{C}}
\newcommand{\calS}{\mathcal{S}}
\newcommand{\calR}{\mathcal{R}}

\DeclareMathAlphabet{\mathsfsl}{OT1}{cmss}{m}{sl}

\renewcommand{\phi}{\varphi}

\newcommand{\Rspace}[1]{\mathbb{R}^{#1}}

\newcommand*{\defeq}{\mathrel{\vcenter{\baselineskip0.5ex \lineskiplimit0pt
                     \hbox{\scriptsize.}\hbox{\scriptsize.}}}%
                     =}

\newcommand{\Prob}[1]{\mathbb{P}\left\{ {#1} \right\}}
\newcommand{\Expect}{\operatorname{\mathbb{E}}}

\theoremstyle{plain}  % Plain style for theorem, defn, lemma, proposition, corollary
\newtheorem{theorem}{Theorem}[section]
\newtheorem{definition}[theorem]{Definition}

\newtheorem{lemma}[theorem]{Lemma}
\newtheorem{proposition}[theorem]{Proposition}

\newtheorem{assumptionA}{Assumption}
\newtheorem{assumptionB}{Assumption}
\newtheorem{assumptionC}{Assumption}
\newtheorem{assumptionD}{Assumption}

\begin{document}



\twocolumn[

\aistatstitle{A Statistical Learning Take on the Concordance Index for Survival Analysis}

\aistatsauthor{ Alex Nowak-Vila \And Kevin Elgui \And Genevi\`eve Robin  }

\aistatsaddress{Owkin Inc. \And Owkin Inc. \And Owkin Inc.} ]

\begin{abstract}
The introduction of machine learning (ML) techniques to the field of survival analysis has increased the flexibility of modeling approaches, and ML based models have become state-of-the-art. These models optimize their own cost functions, and their performance is often evaluated using the concordance index (C-index). From a statistical learning perspective, it is therefore an important problem to analyze the relationship between the optimizers of the C-index and those of the ML cost functions. We address this issue by providing C-index Fisher-consistency results and excess risk bounds for several of the commonly used cost functions in survival analysis. We identify conditions under which they are consistent, under the form of three nested families of survival models. We also study the general case where no model assumption is made and present a new, off-the-shelf method that is shown to be consistent with the C-index, although computationally expensive at inference. Finally, we perform limited numerical experiments with simulated data to illustrate our theoretical findings.
\end{abstract}

\section{INTRODUCTION}

Survival analysis \citep{Gross1981SurvivalMA, Kalbfleisch2002}, the field of statistics concerned with modeling time-to-event data, is central to healthcare applications to predict time from diagnosis to death or risk of disease recurrence. Rather than directly modeling time-to-event, many survival models predict risk of event occurrence \citep{haider2020effective}.  Many definitions of risk can be found in the literature; the most classic are the expected time-to-event, the probability of an event occurring after a given time, or the multiplicative factor in the hazard rate under the proportional hazards (PH) assumption. Importantly, survival data are often \textit{right-censored}, and only a lower bound on the time-to-event is observed; it usually corresponds to the time at which patients leave the study. Most classical survival models have therefore been extended to the censored case (see, e.g.~\citet{klein03} Chapter 3 for a review of the different types of censoring and Chapter 4 for survival estimation in the censored case). 

Machine learning models are increasingly used in survival analysis and have shown state-of-the-art results in various application areas~\citep{zhu_deep_2016, zhu_wsisa_2017, yousefi_predicting_2017, katzman_deepsurv_2018, ching_cox-nnet_2018, kvamme_time_event_2019, barnwal2020survival, steingrimsson2020deep, IDnetwork_cottin, schutte_artificial_2022}. Evaluating this jungle of risk models is therefore an important issue to make it comprehensive for practitioners~\citep{Park2021ReviewOS}. Among existing metrics in survival analysis, the concordance index (C-index) is probably the most commonly used~\citep{Harrell1996MultivariablePM}. It can be viewed as an extension of the Area Under the ROC Curve (AUC) for continuous outcomes and assesses the ability of a risk prediction method to correctly rank individuals according to their risk scores. More specifically, it is defined as the probability that pairs of predicted risks are ranked in the same order as the corresponding observed time-to-events. %Its appeal comes both from its independence from the particular definition of risk and from its probabilistic interpretation. %Specifically, a C-index of 0.5 indicates that the model orders patients randomly, whereas a model that achieves a C-index approaching 1.0 indicates an underlying ordering of the patient cohort.

From a statistical learning perspective, the question arises whether the C-index can be directly optimized, i.e., used as an objective function to be maximized in an ML approach. Unfortunately, the C-index is a non-concave, discontinuous loss with respect to the parameters of the risk model; consequently, gradient-based methods cannot be used directly to maximise it. In practice, models are often learned by minimizing a smooth surrogate loss on the training data and then evaluated using the C-index. Examples of training losses include the negative log-likelihood of survival models such as~\cite{coxPH} Proportional Hazards (PH) or Accelerated Failure Time (AFT) models~\citep{weiAFT}, loss functions defined as the expectation of an error measure between the time to event and the risk predictor \citep{steingrimsson2020deep}, or smooth approximations of the negated C-index \citep{chen2013gradient}.

Despite the widespread use of the C-index as an evaluation measure, the relationship between optimizers of these training losses and those of the C-index is not well understood. In particular, it is not known under what conditions \textit{Fisher consistency}~\citep{FisherOnTM}, also known as classification-calibration~\citep{bartlett2006convexity}, holds, i.e., minimizers of the training loss correspond to optimizers of the C-index. If this property holds, we can safely say that the ML model converges to the optimal C-index as sample size grows to infinity—if the model is expressive enough. %As we show in \cref{sec:fisherconsistency}, this question is much harder than in the case of the AUC metric~\citep{clemenccon2008ranking, agarwal2014surrogate}.

The aim of this paper is to answer this very question. We study the consistency properties of classical cost functions in survival analysis with respect to the C-index, and provide associated excess risk bounds. We analyze in particular the properties of Maximum Likelihood Estimation (MLE), conditional average risk estimation, and smooth C-index maximization. We identify conditions under which these methods are consistent, under the form of three nested families of survival models. In addition, we study the more general case where no model assumption is made. In this case, we present a new, off-the-shelf convex method that is shown to be consistent with the C-index, although computationally expensive at inference. Finally, we perform limited numerical experiments with simulated data to illustrate our theoretical findings. In all cases, we discuss how censoring can be incorporated in our results. Note that, most of the theoretical results can be applied beyond survival analysis to any continuous ranking task in the sense of \cite{clemenccon2018ranking}. Specifically, the following contributions are provided:
\begin{itemize}
    \setlength\itemsep{1em}
    \item The properties of commonly used risk estimation procedures in survival analysis are analyzed in terms of Fisher-consistency and C-index excess risk bounds. 
    \item Conditions under which these procedures are Fisher-consistent with respect to the C-index are derived in the form of three nested families of survival models corresponding to increasingly stringent model assumptions. For each family, we characterize the maximizers of the C-index, and provide important examples.
    \item We discuss a novel, off-the-shelf convex estimation method which, although computationally expensive at inference, proves consistent without any modeling assumption. 
    \item Limited experiments are conducted with simulated data to
    illustrate our theoretical findings.
\end{itemize}

\noindent\textbf{Related Work}

This work is in line with an extensive literature on the statistical efficiency of minimizing surrogate losses of nonconvex and discontinuous evaluation metrics such as the 0-1 loss.~\citet{bartlett2006convexity} derives upper bounds on the excess risk of convex surrogates for binary classification, while~\cite{agarwal2014surrogate} provides similar results for surrogates of bipartite ranking losses.~\cite{cortes2003auc} provides a statistical analysis of the relationship between AUC and error rate minimization, and ~\cite{auc_consistency_pairwise} identifies sufficient conditions for consistency of pairwise surrogate losses with the AUC. From an optimization point of view, ~\cite{calders2007efficient} suggests maximizing the AUC directly using polynomial approximations.

While the C-index is widely used in survival analysis and several papers have investigated its properties from a practical point of view~\citep{longato_practical_2020, Park2021ReviewOS}, there are comparatively few statistical learning results evaluating its relationship to commonly used cost functions used in survival analysis.~\cite{steck2007ranking} provide lower bounds on the C-index that can be directly optimized and examine their relationship to Cox's proportional hazards, showing that this popular model approximately maximizes the C-index. The authors do not examine consistency.~\cite{chen2013gradient} develop a gradient-boosting procedure to optimize smooth surrogates of the C-index; however, the statistical consistency of such surrogates remains to be analyzed.

Since the C-index is fundamentally a ranking measure, our work has similarities with the extensive literature on ranking algorithms and their statistical properties~\citep{clemenccon2008ranking, duchi2010consistency, chapelle2011future, rajkumar2014statistical, yuan2016druge, he2018drug, ai2019learning, wu2021neurank, werner2021review}. As far as we are aware, none of these papers analyze ranking algorithms in terms of consistency with the C-index. \cite{clemenccon2018ranking} examines the optimizers of the C-index, but they focus only on the case where the conditional cumulative distribution functions do not cross one another; we consider much less restrictive assumptions. 

\section{SETTING AND BACKGROUND}
\label{sec:setting}

%The setting and notation specific to survival analysis and to the basic properties of the C-index are presented below, before formalizing the problem of finding consistent C-index optimizers in terms of Fisher consistency. 

\subsection{Survival Analysis}\label{subsec:survivalanalysis}
Consider the classical survival analysis framework to model time-to-events and their relationship to individual covariates. The time-to-event, denoted $T$, is assumed continuous and takes values in $\Rspace{}_+$; individual covariates are denoted by $X$ and take values in $\mathcal{X}\subset\Rspace{d}$. 
% All densities are defined with respect to the probability space $\left(\Rspace{}_+\times \calX, \mathcal{B}(\Rspace{}_+\times \calX)\right)$—where $\mathcal{B}(\Rspace{}_+\times \calX)$ is the Borel $\sigma$-algebra constructed from $\Rspace{}_+\times \calX$—equipped with the Lebesgue measure. 
Let $\PP = \operatorname{Prob}(\Rspace{}_{+}\times \calX)$ be the space of joint probability densities~$\mu$ of non-negative time-to-events and covariates; $\mu(t,x)$ is also referred to as the \textit{survival model}. 
The density of events conditional on covariates $x$ is denoted $\density{t}{x}$, and the conditional survival function
\begin{equation}
    \label{eq:survival-curve}
    \survival{t}{x} = \Prob{T>t|X=x} = \int_t^{+\infty}\mu(t|x)dt.
\end{equation}
We also consider the right-censored setting where the time-to-event $T$ is not directly observed but rather a lower bound~$U = C \wedge T$ where~$C$ is a continuous, nonnegative random variable corresponding to the censoring time. The binary random variable $\Delta = 1(C\geq T)$ specifying whether the lower bounds corresponds to the time-to-event or to the censoring time is also observed. Throughout this paper, the censoring is assumed independent of the covariates, i.e. $C \perp X$, and the censoring curve is defined by $G(t) = \Prob{C>t}$.

\subsection{Concordance Index}\label{sec:cindex}

Consider a scalar valued function 
$$f: x\in\Rspace{d}\mapsto f(x)\in\Rspace{};$$
$f$ may for instance come from an inference procedure assessing the risk of occurrence of events, depending on covariates $x$. In such settings, the higher $f(x)$, the smaller the time-to-event. Many quantities may be used to define the risk; for instance, the conditional expectation of the time-to-event $\Expect{}\{T|X=x\}$, the probability of an event occurring beyond a certain time $t_0$, $\Prob{T>t_0|X=x}$, or a statistic specific to a survival model such as the multiplicative factor of the baseline hazard under the PH assumption.  

The C-index is defined as the probability of having a pairwise concordant order between the predicted risks and the observed time-to-events \citep{harrell1982evaluating}. It is usually presented as the following conditional probability:
\begin{equation}\label{eq:cindex}
    \Cind(f) = \Prob{f(X) < f(X')~|~T>T'}.
\end{equation}
The C-index depends on the joint distribution of $(T,X)$, so that two survival models $\mu$ and $\mu'$ yield two different definitions of the C-index $C_{\mu}$ and $C_{\mu'}$. Whenever the model $\mu$ is clear from context, we drop the subscript for ease of notation. Since the C-index measures the quality of the ranking induced by $f$ rather than the risk values themselves,
it is defined up to monotone transformation of the risk.

For any pair of real random variables $Y$ and $Z$, the \textit{statistical preference} order~\citep{taplin1997preference} denoted by $\succeq$, is defined as
\begin{equation}
\label{eq:statistical-preference}
    Y \succeq Z \iff \Prob{Y>Z} \geq \frac{1}{2}.
\end{equation}
A risk function $f$ defining a global ordering maximally preserving the statistical preference~\eqref{eq:statistical-preference} in expectation for all pairs of conditional random variables~$T|X=x$ and~$T|X=x'$ is optimal with respect to the C-index. This follows directly from the fact that: 
\begin{align}
    \Cind(f) &\propto \Prob{f(X)<f(X'), T>T'} \notag \\
    &= \Expect{1(f(X)<f(X'))1(T>T')} \notag \\
    &= \Expect_{X,X'}{\Prob{T>T'|X,X'}1(f(X)<f(X'))} \label{eq:mwfas-expression}.
\end{align}
%In other words, if it exists, the optimizer of the C-index is a global ordering respecting as much as possible the statistical preference order in expectation. 
\Cref{def:optimal-risk-ordering} presents the case where there exists a global ordering respecting \textit{all} pairwise comparisons under \eqref{eq:statistical-preference}.

\begin{definition}[Optimal risk ordering]\label{def:optimal-risk-ordering} An \textit{optimal risk ordering} is a function $f^\star$ satisfying 
\begin{equation}\label{eq:fullrankingcondition}
    f^\star(x) \leq f^\star(x') \Rightarrow \Prob{T>T'|x,x'} \geq \frac{1}{2},
\end{equation}
for all pairs $x, x'$. Note that if condition \eqref{eq:fullrankingcondition} is satisfied then it follows directly that $f^\star$ is an optimizer of the C-index and it only depends on the conditional density of events $\mu(t|x)$.
\end{definition}
We show in \cref{sec:maximizers-cindex} that an optimal risk ordering does not exists in general. 
%In other words, there exists a risk function defined up to monotone transformations that respects the ordering defined by $\Prob{T>T'|x,x'}$.

\noindent\textbf{C-index estimation with right-censored data}

In the survival analysis setting introduced in \cref{subsec:survivalanalysis}, time-to-events $T$ are not observed but rather a lower bound $U = C \wedge T$ \footnote{$a\wedge b = \min(a, b)$.} and an event indicator $\Delta = 1(C\geq T)$. In this case the question arises of how to consistently estimate the C-index of a risk model $f$. Using the Inverse Censored Probability Weighting (IPCW) strategy of \citet{robins2000IPCW}, one can see that 
\begin{align*}
    \Expect{}\{1(T > T')\} &= \Expect\left\{ \frac{1(T > T') 1(C\wedge C'\geq T')}{G(T')^2} \right\} \\
    &= \Expect{}\left\{\frac{\Delta'1(U> U')}{G(U')^2}\right\},
\end{align*}
where $G$ is the censoring curve defined at the end of \cref{subsec:survivalanalysis}. In particular, this leads to the following expression for the C-index as an expectation over $(X, U, \Delta)$:
\begin{equation*}
    C(f) = \Expect{}\left\{\frac{\Delta'1(U>U') 1(f(X)<f(X'))}{G(U')^2}\right\}
\end{equation*}
The C-index can be consistently estimated from data using the empirical average instead of the expectation; this estimator is known as Uno's C-index \citep{uno2011c}. 

\subsection{Fisher Consistency}
\label{sec:fisherconsistency}

As previously discussed, the C-index \eqref{eq:cindex} cannot be maximized using gradient descent. Instead, the function $f$ is learned by minimizing a \textit{smooth} risk $\calR(f)$. Fisher consistency is a property guaranteeing that the minimizers of the smooth risk are also maximizers of the C-index. \Cref{def:consistency} formalizes the notion of Fisher consistency over a family of distributions.

\begin{definition}[Fisher Consistency]\label{def:consistency} The risk $\calR$ is said to be Fisher consistent to the C-index under a distribution family $\Q\subseteq \PP$ if
\begin{equation*}
    \calR_{\mu}(f^\star) = \min_{f}~\calR_{\mu}(f) \implies C_{\mu}(f^\star) = { \max_{f}}~C_{\mu}(f),
\end{equation*}
for all distribution $\mu\in\Q$, 
where $\calR_{\mu}, C_{\mu}$ are the smooth risk and C-index computed over the joint distribution $\mu$. 
\end{definition}
In this paper, pairs of smooth risks and families of distributions are examined for which Fisher consistency holds for the C-index. 

\noindent\textbf{Consistency for AUC with binary labels}
The C-index is a continuous outcome version of the well-known AUC used in binary output ranking problems, also known as bipartite ranking. AUC is defined as the pairwise probability of concordant order between the risk and the binary labels:
\begin{equation}\label{eq:auc}
    \operatorname{AUC}(f) = \Prob{f(X)>f(X')|Y=1,Y'=0},
\end{equation}
where $Y,Y'\in\{0,1\}$ are binary. The loss is non-continuous and cannot be optimized directly by gradient descent. In this case, however, the set of optimizers of \eqref{eq:auc} can be easily characterized as monotone transformations of the conditional distribution~\citep{agarwal2014surrogate}. This can be seen as the function $f(x) = \Prob{Y=1|X=x}$ satisfies  
\begin{equation*}
    \Prob{Y>Y'|x,x'} \geq \frac{1}{2} \iff f(x) \geq f(x'),
\end{equation*}
which directly follows from the identity $f(x)(1-f(x')) = \Prob{Y>Y'|x,x'}$. Hence, any smooth risk $\calR$ whose minimizer is a monotone transformation of the conditional distribution is consistent to the AUC. This includes least squares, logistic regression, and more generally any proper loss function~\citep{agarwal2014surrogate}.

\section{Maximizers of C-index}\label{sec:maximizers-cindex}

%\note{A: Maybe add a simple figure describing the assumptions.} 
The problem of maximizing the C-index associated to a survival model $\mu$ writes
\begin{equation}
    \label{eq:C-maximization}
    C^{\star}_{\mu} = \max_{f} C_{\mu}(f).
\end{equation}
The aim of this section is to classify the set of possible survival models $\mu$—the joint density of $(T,X)$—in terms of properties of the associated maximizers of~\eqref{eq:C-maximization}. We introduce four families of survival models, denoted by $\A \subsetneq \B \subsetneq \C$ and $\D\defeq \C^c = \PP\setminus \C$, defined informally as follows:
\begin{itemize}
\setlength\itemsep{0cm}
    \item $\A$: Survival curves $S(t|x)$ do not cross. Existing work studying consistency with respect to C-index, such as~\citet{clemenccon2013ranking,clemenccon2018ranking} are limited to this family of models.
    \item $\B$: $-\Expect\{T|X=x\}$ is an optimal risk ordering~\eqref{eq:fullrankingcondition}. We prove in \Cref{sec:consistency-excess} that a large family of smooth risk functions is consistent in this family. 
    \item $\C$: There exists an optimal risk ordering satisfying~\eqref{eq:fullrankingcondition}. We prove in \Cref{sec:consistency-excess} that MLE is Fisher-consistent for several examples of models in family $\C$.
    \item $\D$: There is no optimal risk ordering satisfying~\eqref{eq:fullrankingcondition}. We introduce in \Cref{sec:consistency-excess} an off-the-shelf estimation method which proves consistent in this setting, although being computationally expensive at inference.
\end{itemize}
In the following, the four families of survival models are described alongside examples and theoretical results characterizing the associated oracle C-index maximizers. 

\subsection*{A. Conditional Survival Curves do not Cross}
\begin{comment}
    Denote by $\mathcal{M}$ the set of all positive densities defined on $\left(\Rspace{}_+\times \calX, \mathcal{B}(\Rspace{}_+\times \calX)\right)$ equipped with the Lebesgue measure,
and consider the following family of survival models:
\begin{multline}
    \label{eq:familyA}
    \A = \{\mu\in\mathcal{M}; \forall (x,x')\in\calX^2,\\
    t\mapsto S(t|x)-S(t|x') \text{ has constant sign} \}.
\end{multline}
\end{comment}
Family $\A$ is defined as the set of survival models whose conditional survival curves uniformly bound one another. In other words, models $\mu\in\A$ satisfy \cref{ass:A}.
\begin{assumptionA}
\label{ass:A}
For all $(x,x')\in\calX^2$, $t\mapsto S(t|x)-S(t|x')$ has constant sign.
\end{assumptionA}
 Under \Cref{ass:A}, \Cref{th:assumption-A} shows that the negative conditional expectation $-\operatorname{CE}(x) = \Expect{}\{T~|~X=x\}$ satisfies Condition~\eqref{eq:fullrankingcondition}; the proof can be found in the Appendix. 
\begin{theorem}\label{th:assumption-A}
If $\mu\in\A$, the negative conditional expectation
is an optimal risk ordering for the C-index satisfying Condition~\eqref{eq:fullrankingcondition}, thus $C_{\mu}(-\operatorname{CE}) = C^{\star}_\mu$.
\end{theorem}
The two most commonly used survival models, namely Cox PH and AFT, satisfy \Cref{ass:A} and therefore \Cref{th:assumption-A} applies, as discussed below.

\noindent\textbf{Proportional Hazards model.} The hazard function is defined as $h(t) = S'(t)(1 - S(t))^{-1}$, where $S'$ denotes the \ note{time derivative of the survival curve}. In the PH model, the conditional hazard $h(t|x)$ factorizes as $h(t|x) = h_0(t)e^{f(x)}$,
where $h_0$ is the (non-negative) baseline hazard and $f$ is a function of the covariates \citep{coxPH}. This yields
$$S(t|x) = S_0(t)^{e^{f(x)}}, \hspace{0.5cm}S_0(t) = e^{-\int_0^th_0(\tau)d\tau},$$
where $S_0$ is the baseline survival curve. 
It directly follows that survival curves do not cross at any point in time, therefore \Cref{ass:A} is satisfied and \Cref{th:assumption-A} applies. In this example $f(x)$ also defines an optimal ranking; derivations are provided in the appendix to support this claim.
\begin{comment}
    Moreover, the pairwise conditional probabilities can be written in closed form independently of the baseline hazard function $h_0$ as \note{ref}
$$ \Prob{T>T'|x,x'} = \frac{e^{f(x)}}{e^{f(x)} + e^{f(x')}}. $$
From this expression it directly follows that $f(x)$ is an optimal scoring function satisfying condition \eqref{eq:fullrankingcondition} as
$ \frac{e^{f(x)}}{e^{f(x)} + e^{f(x')}} \geq \frac{1}{2}$ if and only if $f(x) \geq f(x')$. 
% Moreover, the maximum C-index can be written as 
% $\Expect_{X,X'}\frac{\min(e^{f(X)}, e^{f(X')})}{e^{f(X)} + e^{f(X')}}$.
\end{comment}

\noindent\textbf{Accelerated Failure Times model.} The AFT model assumes the following form for time-to-events:
\begin{equation}\label{eq:aft-model}
    \log T = f(X) + \varepsilon,
\end{equation}
where $\varepsilon$ is an independent random variable. The survival curve is parametrized as $S(t|x) = S_0(te^{-f(x)})$,
where $S_0$ is the survival curve of $e^{\varepsilon}$. Note that the survival curves do not cross each other as they are defined as scaling by~$e^{f(x)}$. In this example also, \Cref{ass:A} is satisfied and \Cref{th:assumption-A} applies. As for the PH model, note that $f(x)$ also defines an optimal ranking; derivations are provided in the appendix to support this claim.
\begin{comment}
  The pairwise conditional probabilities read
\begin{equation}\label{eq:pairwise-proba-aft}
    \Prob{T>T'|x,x'}  = F_{\varepsilon-\varepsilon'}(f(x) - f(x')),
\end{equation}
where now $F_{\varepsilon-\varepsilon'}$ stands for the cumulative distribution function of the symmetrical random variable $\varepsilon-\varepsilon'$, with both $\varepsilon,\varepsilon'$ following the same distribution (see \note{APPENDIX} for a derivation). From the symmetry at the origin of $\varepsilon -\varepsilon'$ it follows that $f(x)$ is an optimal scoring function satisfying \eqref{eq:fullrankingcondition} as $F_{\varepsilon-\varepsilon'}(f(x) - f(x')) \geq 1/2$ if and only if $f(x) \geq f(x')$.
% Moreover, the minimum C-index can be written as $\Expect_{X,X'}F_{\varepsilon-\varepsilon'}(|f(X') - f(X)|)$. 
\end{comment}

\subsection*{B. Conditional Expectation is an Optimal Ordering}
Family $\B$ is defined as the set of survival models for which the negative conditional expectation is an optimal risk ordering. In other words, models $\mu\in\B$ satisfy the following condition.
\begin{assumptionB}
\label{ass:B}
The negative conditional expectation $-\operatorname{CE}(x) = -\Expect{\{T|X=x\}}$ is an optimal risk ordering satisfying~\eqref{eq:fullrankingcondition}.
\end{assumptionB}
Note that \cref{th:assumption-A} proved the inclusion $\A\subset\B$. Now, the strict inclusion, $\A\subsetneq\B$ is proved by providing an example of survival models such that $\mu\in\B\setminus\A$. Consider the extended AFT model presented in the previous section by adding \textit{symmetric heteroscedastic noise}.
\begin{definition}[AFT-H] In the AFT-H model, the time-to-event has the form
\begin{equation}\label{eq:aft-h-model}
    \log T = f(x) + \sigma(x)\varepsilon,
\end{equation}
where $\sigma:\calX\to \Rspace{}_+$ is a positive-valued function satisfying $f(x)\leq f(x')\implies \sigma(x)\leq \sigma(x')$ and $\varepsilon$ is a centered Gaussian random variable.
\end{definition}
\begin{comment}
    The pairwise conditional probabilities read
\begin{equation}\label{eq:pairwise-proba-afth}
    \Prob{T>T'|x,x'}  = F_{\varepsilon-\varepsilon'}(\sigma^2(x) + \sigma^2(x'))^{-1/2}(f(x) - f(x')),
\end{equation}
where $F_{\varepsilon-\varepsilon'}$ stands for the cumulative distribution function of the symmetrical random variable $\varepsilon-\varepsilon'$, with both $\varepsilon,\varepsilon'$ following the same distribution. 
\end{comment}
\Cref{prop:afth-in-B} shows that AFT-H satisfies \cref{ass:B}.
\begin{proposition}[AFT-H satisfies $\B$]
\label{prop:afth-in-B}
Assume that $\mu$ is in AFT-H. Then, the negative conditional expectation is an optimal risk ordering, thus
$C_{\mu}(-\operatorname{CE}) = C_\mu^{\star}$.
\end{proposition}
This result is a reformulation of Corollary 2 by \cite{lebedev2019nontransitivity}; we provide the original statement in the appendix. 
Note that under AFT-H the conditional survival curves take the following form 
\begin{equation*}
    S(t|x) = S_{\varepsilon}\Big(\frac{\log t - f(x)}{\sigma(x)}\Big),
\end{equation*}
where $S_{\varepsilon}(X) = \Prob{X>\varepsilon}$. Fixing $f(x)$ and varying $\sigma(x)$ we can clearly see how the survival curves cross so that AFT-H \textit{does not} always satisfy \cref{ass:A}.

\subsection*{C. There exists an Optimal Risk Ordering}
Family $\C$ is defined as the set of survival models admitting an optimal risk ordering satisfying~\eqref{eq:fullrankingcondition}.
\begin{assumptionC}
\label{ass:C}
There exists an optimal ordering $f^\star_\mu$ for survival model $\mu$, satisfying~\eqref{eq:fullrankingcondition}.
\end{assumptionC}
Under \cref{ass:C}, a closed form for the maximum C-index attained at $f^\star$ can be derived by combining \cref{prop:maximizer-closed-form} below to the expression of pairwise conditional probabilities for specific models.
\begin{proposition}\label{prop:maximizer-closed-form} Assume that $\mu$ satisfies \cref{ass:C}. Then, the optimal C-index takes the following form:
\begin{equation*}
    C_{\mu}^\star = C_{\mu}(f^\star_\mu) = \Expect_{X, X'}\phi(\Prob{T>T'|X,X'}), 
\end{equation*}
where $\phi(a) = \max(a, 1-a)$ and $f^\star$ satisfies \eqref{eq:fullrankingcondition}.
\end{proposition}
Analogously to the previous cases, a family of distributions satisfying \cref{ass:C} is introduced, using exponential family models. Then, an example of model $\mu\in\C\setminus\B$ is provided.
\begin{definition}[Exponential family survival model]
\label{def:scalar-expo-family}
For $\theta:\calX\to\Rspace{}$, $\beta:\Rspace{}_+\to\Rspace{}_+$, $\tau:\Rspace{}_+\to\Rspace{}$, $\eta:\Rspace{}\to \Rspace{}$, and $A:\Rspace{}\to\Rspace{}$ such that, for all $x\in\calX$, the conditional density of a curved exponential family model is given by
\begin{equation}\label{eq:exponential-survival}
    \mu(t|x) = \beta(t)\exp\left[\eta\circ\theta(x)\tau(t) - A\circ\theta(x)\right],
\end{equation}
with associated parameter $\theta(x)$. For instance, $\theta(x) = \theta^\top x$ in a generalized linear model.  
\end{definition}
The scalar exponential family from \cref{def:scalar-expo-family} covers many of the survival curves classically used in survival analysis, e.g., the exponential, chi-squared, Laplace and normal distributions. Under the exponential family model, the scalar parameterization $\theta(x)$ satisfies the optimal risk ordering condition, as shown in the following proposition proved in the Appendix. Thus, maximization of the C-index can be achieved by estimating the parameters $\theta(x)$ of the model. %\note{A: might be a little bit confusing to use $\theta$ instead of $f$.} 
\begin{proposition}
    \label{prop:exp-fam}
    Under \Cref{ass:C}, with $\theta$ continuous, $\beta$ positive, $\tau$ non-decreasing and $\eta$ continuously differentiable and non-decreasing, $\theta(x)$ is an optimal risk ordering for the C-index, thus
    $C_\mu(\theta) = C_\mu^{\star}$.
\end{proposition}

\noindent\textbf{Weibull with varying shape parameter}
Consider the model defined by Weibull conditional survival curves with varying shape parameter
\begin{equation}\label{eq:weibull-model}
    S(t|x) = e^{-t^{f(x)}}.
\end{equation}
The following \cref{prop:weibull} proves the strict inclusion $\B\subsetneq\C$.
\begin{proposition}\label{prop:weibull}
    The above Weibull model \eqref{eq:weibull-model} satisfies \cref{ass:C} but not \cref{ass:B}.
\end{proposition}
The proof is based on the result by \citet{lebedev2019nontransitivity} showing that~$f$ gives an optimal risk ordering. \Cref{ass:B} is not satisfied as the expectation of a Weibull is not monotone on the shape parameter. Indeed, the expectation is given by $\Gamma(1 + \frac{1}{f(x)})$, where $\Gamma$ denotes the Gamma function. This function has a minimum between $1.46$ and $1.47$, decreasing first and increasing for larger values. Thus, it gives a different ranking than the optimal risk ordering $f$.

\subsection*{D. There is no Optimal Risk Ordering}\label{subsec:notorder}

Family $\D$ contains survival models for which there does not exist an optimal risk ordering satisfying condition \eqref{eq:fullrankingcondition}. The following example illustrates this phenomenon. Let $\{T_i\}_{1\leq i \leq n}$ be random variables corresponding to time-to-event of individuals $1\leq i \leq n$, and consider the following assumption.
\begin{assumptionD}
\label{ass:D}
There exists $m\geq 3$, a subset of indices $\mathcal{I}\subset\{1,\ldots,n\}$, $|\mathcal{I}|=m$, and an ordering $i_1< i_2< \ldots< i_m$ such that, denoting $i_{m+1} = i_1$,
    \begin{equation*}
    \min_{k\in\{1,\ldots,m\}}\left(\Prob{T_{i_k}<T_{i_{k+1}}}\right) > \frac{1}{2}.
\end{equation*}
\end{assumptionD}
\Cref{ass:D} implies the existence of a cyclic sequence with respect to the statistical preference order~\eqref{eq:statistical-preference}, which implies there is no ranking function satisfying the optimal risk ordering~\eqref{eq:fullrankingcondition}. 
\begin{figure}[t!]
    \centering
    \includegraphics[width=0.5\textwidth]{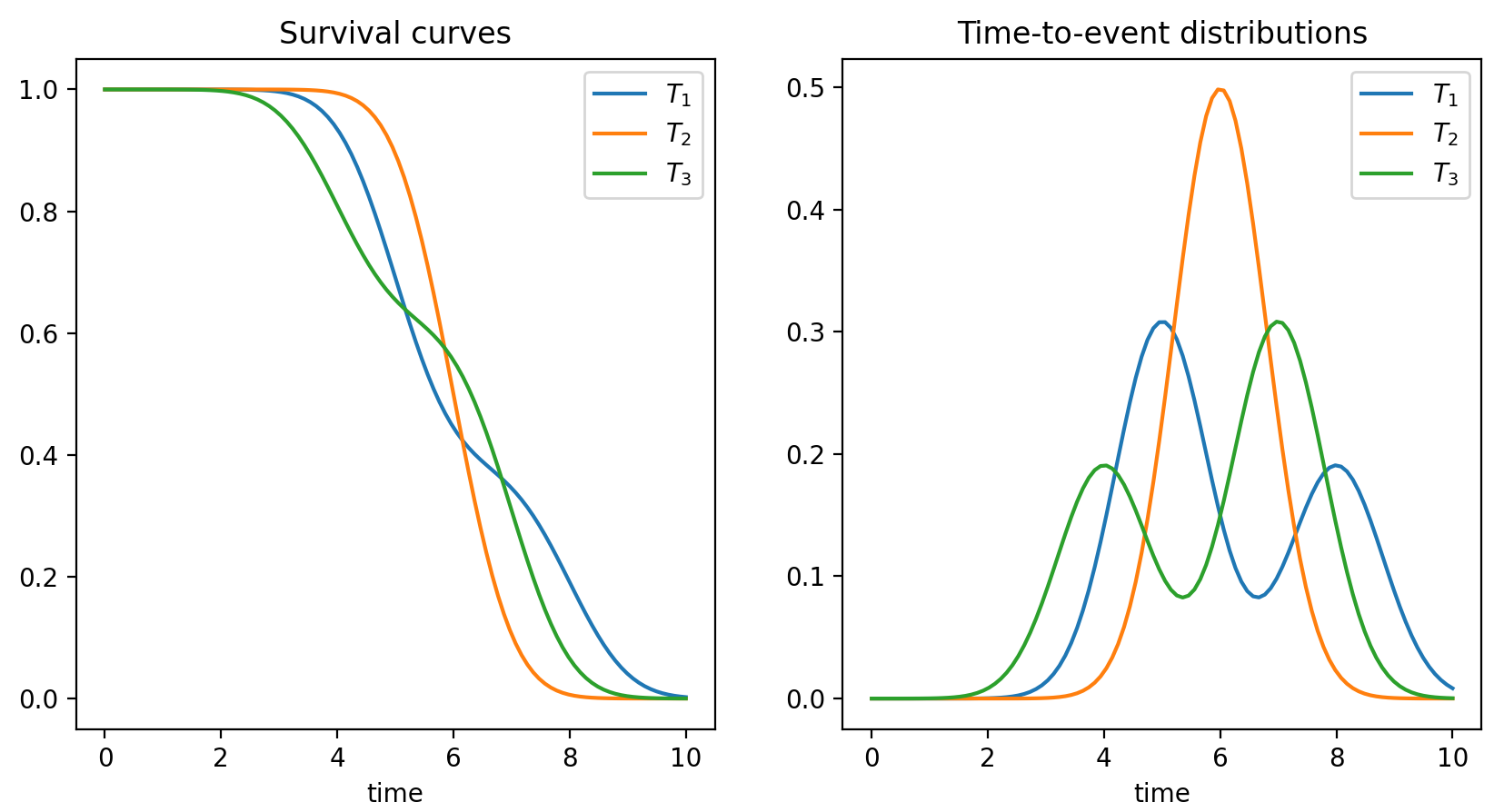}
    \caption{Survival curves and time-to-event distributions of a cycle of length three. In this case $P(T_1 < T_2) = P(T_2 < T_3) \approx 0.52$ and $(T_3 < T_1) \approx 0.55$. Hence, it is a cycle as $\min(P(T_1 < T_2), P(T_2 < T_3), P(T_3 < T_1)) > \frac{1}{2}$. }
    \label{fig:no_transitivity}
\end{figure}
In \cref{fig:no_transitivity}, we illustrate this phenomenon with a cyclic sequence made of a uni-modal and two multi-modal time-to-event distributions. 

In the previous sections, we showed that for $\mu \in \A,\B,\C$, the maximizer of the C-index in fact only depends on the \textit{conditional density} $\mu(t|x)$; however, if $\mu\in\D$, the maximizer may also depend on the marginal covariate distribution $\mu(x)$. This is an important characteristic of cyclic sequences, since the optimizer of the C-index may change under distributional shifts of the marginal population of patients. This phenomenon is shown in the following result, proved in the Appendix. 
\begin{proposition}\label{prop:marginal-dependence}
    Under \Cref{ass:D}, the maximizer of the C-index depends on the marginal distribution of the patients covariates $\mu(x)$. 
\end{proposition}
In particular, this means that the optimal relative order between patients may change if new patients are added to the original patient cohort. This phenomena does not happen in the binary setting where the ranking measure is the AUC as its optimizer is the conditional expectation, as discussed in \cref{sec:fisherconsistency}. 
%It does not happen either for cost function defined as expectations over only input-output pairs as $\Expect~L(f(X), Y)$, which cover the wide majority of cost functions used in practice \citep{bartlett2006}.

\section{CONSISTENCY AND EXCESS RISK BOUNDS}
\label{sec:consistency-excess}

\begin{figure*}[ht]
    \centering
    \includegraphics[width=.66\columnwidth]{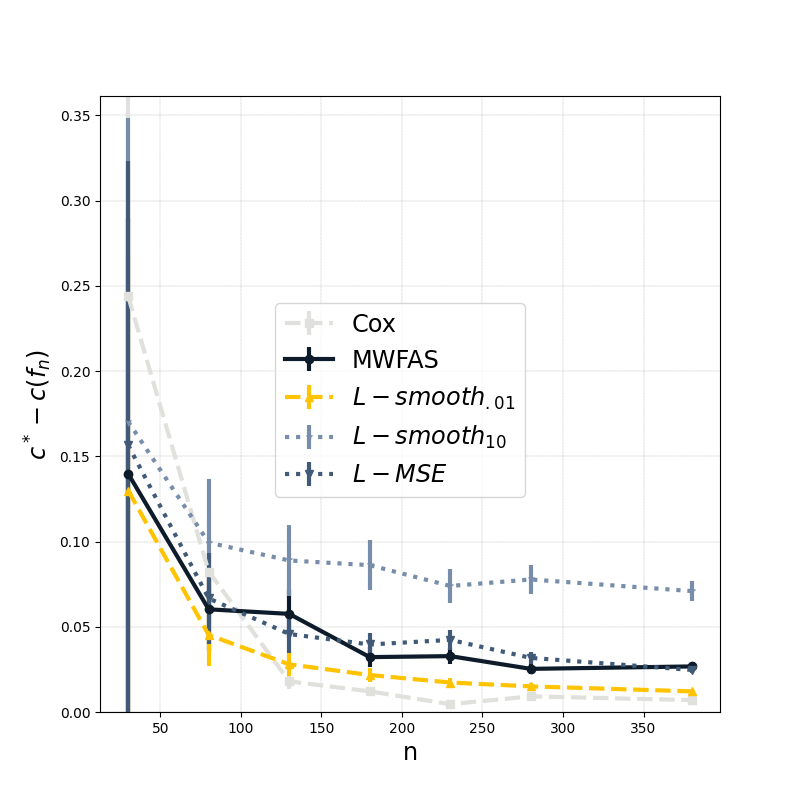}
    \includegraphics[width=.66\columnwidth]{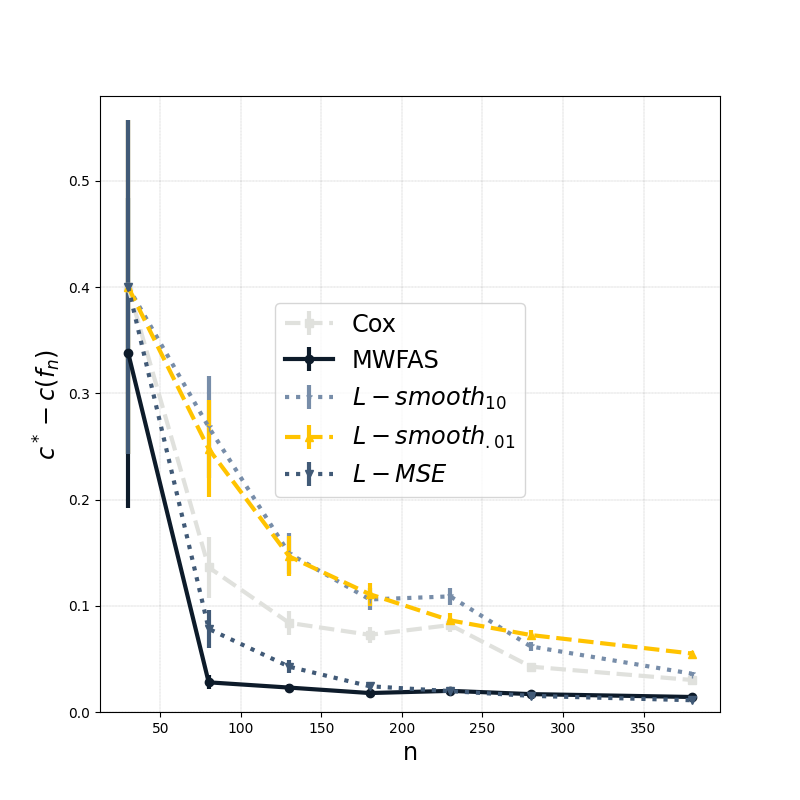}
    \includegraphics[width=.66\columnwidth]{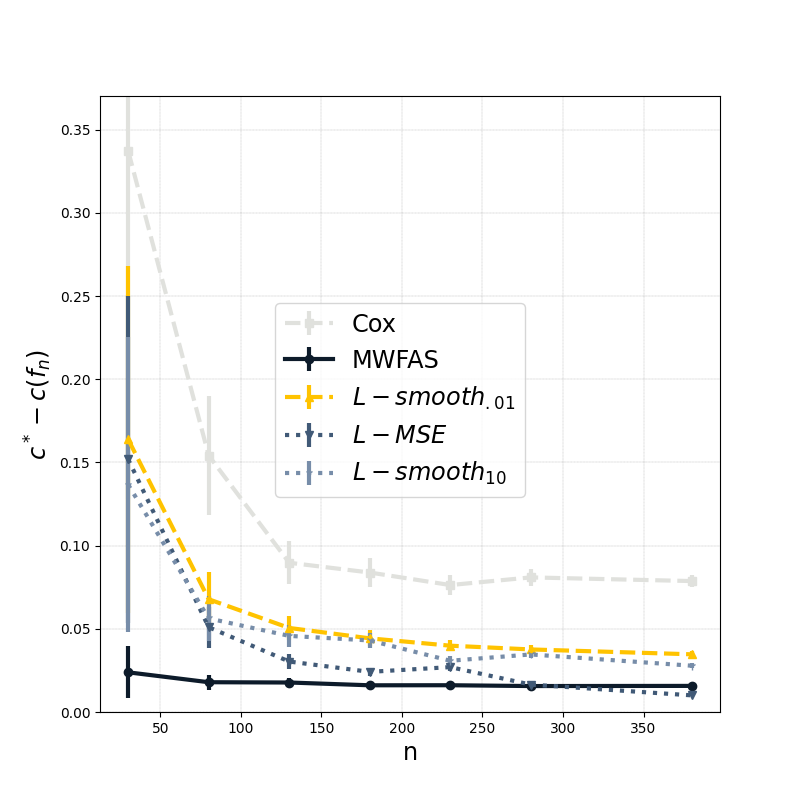}
    \caption{C-index empirical excess risks of the different methods of ranking for various training sizes $n$. The left, center and right figures correspond to data generated from models satisfying assumptions A, B and C, respectively. }
    \label{fig:results}
\end{figure*}

We now study several estimation procedures classically used in survival analysis, all based on Fisher-consistent, smooth cost functions. In particular, we discuss under which families of survival models introduced in the previous section Fisher-consistency holds. Although existing work was limited to family $\A$ (e.g.~\cite{clemenccon2013ranking, clemenccon2018ranking}), we prove that all the considered methods are consistent in family $\B \supsetneq A$. We start by introducing the considered method and their extensions to the censoring case. Then, we provide excess risk bounds on the associated C-index suboptimality. 

\subsection{Estimation Procedures}
\label{sec:survival_training_losses}

\noindent\textbf{Estimating the conditional expectation (A, B).} Without a specific survival model but under \cref{ass:B}, one can use any cost function minimized by a monotone transformation of the conditional expectation. The following \cref{th:cost-functions-B} provides a family of risks based on Fenchel-Young losses \citep{blondel2019learning} satisfying this property. 

\begin{theorem}\label{th:cost-functions-B}
    Let $\Omega:\calC\rightarrow\Rspace{}$ be a twice-differentiable \textit{strongly convex} function \footnote{A one-dimensional strongly convex function is one for which the Hessian is uniformly lower bounded $\nabla_u^2\Omega \geq C> 0$ for all $u$ in the domain.} defined in a closed domain $\calC\supseteq \Rspace{}_{+}$ such that $\lim_{u\to\infty}\nabla\Omega(u) = +\infty$. Define the cost function $S(v, t) = \Omega^{*}(v) - vt$ where $\Omega^{*}$ is the Fenchel conjugate of $\Omega$ \citep{rockafellar1997convex} \footnote{The Fenchel conjugate $\Omega^{*}$ of $\Omega$ is defined for all $v\in\Rspace{}$ as $\Omega^*(v) = \sup_{u\in\calC}~vu - \Omega(u)$}. Then, the following risk
\begin{equation}\label{eq:cost-function-conditional}
    \calR(f) = \Expect_{(X, U, \Delta)}~\frac{\Delta S(f(X), U)}{G(U)},
\end{equation}
is convex, smooth, and its minimizer is a monotone transformation of the conditional expectation. Thus, under \cref{ass:B} it is Fisher consistent to the C-index. 
\end{theorem}

When $\calC = \Rspace{}$ and $\Omega(u) = u^2/2$ this corresponds precisely to ICPW \citep{robins2000IPCW} on least squares. However, \Cref{th:cost-functions-B} provides a larger family of consistent smooth risks by choosing $\Omega$ and its domain $\calC$ using the construction of Fenchel-Young losses. 

Note that the estimator minimizing \eqref{eq:cost-function-conditional} is not efficient in the presence of censoring as only samples corresponding to events $\Delta=1$ contribute to the cost function. To alleviate this issue, \citet{steingrimsson2020deep} uses semi-parametric efficiency theory for missing data \citep{tsiatis2006semiparametric, robins1994estimation}, and develops an augmented estimator with smallest asymptotic variance among all unbiased estimators of $\Expect~S(f(X), T)$. %However, this approach is not very useful as to compute the augmented term one has to estimate the conditional survival curve, which is a more complicated object than the one we want to estimate.  

\noindent\textbf{Maximum Likelihood Estimators (A, B, C).} Assume the conditional survival model $\mu_f(t|x)$ lies in a family of distributions so that one of the parameters $f(x)$ gives the optimal risk ranking \eqref{eq:fullrankingcondition}, thus belonging to class $\C$. We can learn the optimal parameter under censoring by minimizing the following MLE loss \citep{Kalbfleisch2002}:
\begin{equation*}\label{eq:mle-cost-function}
   \calR(f) = -\Expect_{(X, U, \Delta)}~\Delta\log \mu_f(U|X) + (1-\Delta)\log S_f(U|X).
\end{equation*}
The derivation of this loss can be found in the Appendix. 
This loss can be used for all models presented in the previous section, namely PH, AFT, AFT-H, Weibull and the exponential family. Whenever the model is identifiable, the MLE is consistent to the true parameter and thus Fisher consistent to the C-index. An important advantage of this loss is that the censored samples have an explicit role and provide signal during the learning procedure. Note that, penalized versions of~\eqref{eq:mle-cost-function} can also be used to obtain explicit finite sample risk bounds (see \Cref{prop:finite-sample-Lasso-excess}).

\noindent\textbf{Smooth C-index (A, B, C, D).} In family $\D$, one cannot assume the existence of an optimal risk ordering. In this case, an alternative method is smooth C-index maximization, which is based on smoothing the indicator function defining the C-index under censoring as presented in \cref{sec:cindex}, leading to a non-convex smooth loss \citep{mayr2016boosting}. The non-convexity of this approach may cause convergence problems and a scaling parameter may be tuned to guarantee proper convergence. 

\noindent\textbf{Estimating pairwise probabilities (A, B, C, D).} The smooth C-index has two-main problems: its non concavity and the fact its optimizer is not robust to marginal distribution shifts due to \cref{prop:marginal-dependence}. We propose a novel methodology whereby instead of learning a risk function $f$ specifying the ranking on the training cohort, we (1) learn the pairwise conditional probabilities $\Prob{T>T'|x, x'}$ on the training data with an estimator $h(x, x')$ and then (2) we construct the ranking that better satisfies the relative order constraints in expectation over the finite validation cohort $x_1,\ldots, x_n$ by solving:
\begin{equation}\label{eq:inference-mwfas}
    \min_{\sigma\in\calS}\sum_{i, j=1}^n\gamma_{ji}1(\sigma(i) < \sigma(j)),
\end{equation}
where $\calS$ is the set of permutations of size $n$ and $\gamma_{ij} = |2h(x_i, x_j) - 1|1(h(x_i, x_j) > 1/2)$. The derivation of this methodology follows easily from the C-index expression \eqref{eq:mwfas-expression} and it is consistent by construction. The combinatorial problem \eqref{eq:inference-mwfas} is known as the Minimum Weight Feedback Arc Set (MWFAS) problem \citep{duchi2010consistency, karp1972reducibility}; it is known to be NP-Hard. However, multiple approximations of this problem exist \citep{even1998approximating, demetrescu2003combinatorial}.

This method addresses the two problems of the smooth C-index. First, the pairwise conditional probabilities can be learned using convex estimation methods, such as logistic regression on the binary problem $\tilde{Y} = \operatorname{sign}(T - T')$ and $\tilde{X} = (X, X')$. Second, the ranking estimator is robust to marginal distributional shifts as the inference algorithm \eqref{eq:inference-mwfas} is computed on the validation cohort. It is interesting to note that the computational bottleneck of the smooth C-index coming from its non concavity has now been transposed into the computational bottleneck of the combinatorial inference problem. 

\subsection{Excess Risk Bounds}

The question is now to obtain excess risk bounds on the C-index when there exists an optimal risk ranking $f^\star$ satisfying \eqref{eq:fullrankingcondition}. The following \cref{th:excess-risk-bound} bounds the excess risk of the C-index.
\begin{theorem}[Excess risk bounds]\label{th:excess-risk-bound}
    Let $f^\star$ be an optimal risk ranking satisfying \eqref{eq:fullrankingcondition}. Let $L> 0$ a positive constant satisfying 
    \begin{equation}\label{eq:bound-L}
        |2\Prob{T>T'|x, x'} - 1| \leq L|f^\star(x) - f^\star(x')|, 
    \end{equation}
    for all pairs $x, x'$. Then, the excess risk of the C-index can be bounded as 
    \begin{equation*}
        C(f^\star) - C(f) \leq 2L\Expect_{X}|f(X) - f^\star(X)|. 
    \end{equation*}
\end{theorem}
The proof of this result can be found in the Appendix, and is based on a reduction of the problem to a binary classification problem with input $(X, X')$ and output $Z=\operatorname{sign}(T - T')$, similar to \cite{agarwal2014surrogate} for the bipartite ranking setting. 
The following \cref{prop:constant-L} shows that condition \eqref{eq:bound-L} is satisfied for most of the models presented in \Cref{sec:maximizers-cindex}. The proof is in the Appendix.

\begin{proposition}\label{prop:constant-L}
    Condition \eqref{eq:bound-L} is satisfied by:
    \begin{itemize}
        \item[(i)] The PH model with $L=1$.
        \item[(ii)] The AFT model \eqref{eq:aft-model} with $L$ the Lipschitz constant of the cumulative distribution function $F_{\varepsilon-\varepsilon'}$ of the symmetric random variable $\varepsilon-\varepsilon'$.
        \item[(iii)] The AFT-H model \eqref{eq:aft-h-model} with $L$ the same as AFT scaled by a factor $a$ where $\sigma(x) \geq 1/a$ for all $x$. 
        \item[(iv)] The exponential family \eqref{eq:exponential-survival}.
    \end{itemize}
\end{proposition}
We now provide two examples of application of \cref{prop:constant-L}.

\noindent\textbf{Example of Lasso estimator on Cox PH model }
\Cref{prop:constant-L} shows that \Cref{th:excess-risk-bound} applies in particular to the PH model. For this specific case, %and in the more simple setting without censoring, 
we illustrate our theoretical findings with an application to the Lasso estimator for the Cox PH model. This estimator is analyzed in~\citet{huang2013oracle}, where finite sample bounds on the $\ell_1$-penalized negative log-likelihood estimation and prediction errors are proven. More specifically, assuming a generalized linear model $f^\star(x) = {\theta^\star}^\top x$, the Lasso estimator $\hat{\theta}_{\operatorname{Lasso}}$ is shown to satisfy $\|\hat{\theta}_{\operatorname{Lasso}} - \theta^\star\|_1\lesssim \frac{\|\theta^\star\|_0\log(d)}{n}$ in high probability for sufficiently large training datasets. Operator $\lesssim$ denotes inequality up to log and constant terms depending on the model, and $\|\theta^\star\|_{0}$ denotes the number of non-zero entries of $\theta^\star$. Combining this result with \Cref{th:excess-risk-bound}, and notincing that 
$$\Expect_X{|(\hat{\theta}_{\operatorname{Lasso}}-\theta^\star)^\top X|}\leq \|\hat{\theta}_{\operatorname{Lasso}}-\theta^\star\|_1\Expect_X{\|X\|_{\infty}},$$
we obtain the following informal result.
\begin{proposition}[\textbf{informal}]
    \label{prop:finite-sample-Lasso-excess}
    With probability at least $1-\varepsilon(n)$, $\varepsilon(n)\to 0$ as $n\to+\infty$, 
    $$C(\theta^\star)-C(\hat{\theta}_{\operatorname{Lasso}})\lesssim \Expect_X{\|X\|_{\infty}}\cdot\frac{\|\theta^\star\|_0\log(d)}{n}.$$
\end{proposition}
The quantity $\Expect_X{\|X\|_{\infty}}$ depends on the covariates model and corresponds to ``the size'' of the input space. 

\noindent\textbf{Finite sample bounds for family $\boldsymbol{\B}$.} Under \cref{ass:B} and using the smooth cost function \eqref{eq:cost-function-conditional} we can obtain excess risk bounds on the C-index in terms of the excess of the smooth risk $\calR(f)$.
\begin{theorem}\label{th:excess-risk-conditional-expectation}
   Let $\calR$ be the risk defined in \eqref{eq:cost-function-conditional}.  Under \cref{ass:B}, the following inequality holds:
    \begin{equation*}
        C^\star - C(f) \leq 4L\gamma\sqrt{\calR(f) - \calR^\star},
    \end{equation*}
    where $\Omega''(u)\geq 1/\gamma^{2}$ for all $u$ in the domain $\calC$ \footnote{Recall that $\Omega$ is convex thus $ \Omega''(u)\geq 0$}.
\end{theorem}

Combining the above \cref{th:excess-risk-conditional-expectation} with finite sample bounds on the excess risk such as the ones obtained by \citet{ausset2019empirical} one can translate them to the C-index.

\section{EXPERIMENTS}

In this section we perform experiments to validate our theoretical findings \footnote{Code can be found in \url{https://github.com/owkin/owkin-metric}}. More specifically, we assess empirically the consistency of different estimation methods with respect to the C-index, under simulation regimes corresponding to families $\A$, $\B$ and $\C$.

\noindent \textbf{Data Generation and Evaluation Procedure}
We simulate survival data using the three different regimes A, B and C presented in \Cref{sec:maximizers-cindex}. We first simulate the training covariates $x_1, \ldots, x_n \in \Rspace{d}$ as well as a unit vector $\beta \in \Rspace{d}$ with $d=10$ and parameterize the optimal ranking linearly $f(x) = \beta^\top x$. We then simulate the corresponding time-to-events $t_1, \ldots, t_n$ as realizations of the distribution of $T|X$, parameterized by $\beta$ and depending on the selected regime of simulation. Four different models are fit using the the cost functions studied in \Cref{sec:consistency-excess}.
First, a linear model (\textbf{L-MSE}) optimizing a mean square error loss function to regress the conditional expectation of $T$ given $x;$ second, a linear Cox model (\textbf{Cox}) optimizing a log-likelihood; third, a linear model (\textbf{L-smooth}$_{\mathbf{\sigma}}$) optimizing a smooth C-index (with smoothing parameter $\sigma \in \{.01, 10\}$ that accounts for the smoothness of the approximation of the indicator function), and finally a pairwise model (\textbf{MWFAS}) that predicts pairwise probabilities using XGBoost, and from which the ranking is obtained by solving the MWFAS combinatorial problem presented in 
\Cref{sec:survival_training_losses} with a fast approximation algorithm.
We compute the obtained C-index from a test dataset of fixed size ($n_{\text{test}} = 3000$) using the same distribution as the training dataset.

\noindent \textbf{Results} The results are provided in \Cref{fig:results}. 
%In these figures, we show the excess risk of each model when $n$ varies, and for numerous dimensions of the vector of covariates d=10 (leaving the other cases in the Appendix). 
For the three generation regimes, we observe that the proposed method \textbf{MWFAS} yields the best performance and converges to the optimal C-index when $n$ is sufficiently large. It has similar performance to \textbf{Cox} in the regime $\A$ where \textbf{Cox} is well specified. When \textbf{Cox} is not well-specified it does not converge to the optimal ranking. Lastly, note that the value of the smoothing parameter used by \textbf{L-smooth}$_{\sigma}$ can harm the performance of the resulting model and it is important to choose it properly to improve convergence guarantees. 

\subsubsection*{Acknowledgements}
The authors would like to thank Paul Trichelair for his valuable guidance throughout the project. 

%\textcolor{red}{\lipsum[1-1]}
% \subsubsection*{Acknowledgements}
% All acknowledgments go at the end of the paper, including thanks to reviewers who gave useful comments, to colleagues who contributed to the ideas, and to funding agencies and corporate sponsors that provided financial support. 
% To preserve the anonymity, please include acknowledgments \emph{only} in the camera-ready papers. 

% \subsubsection*{References}

% References follow the acknowledgements.  Use an unnumbered third level
% heading for the references section.  Please use the same font
% size for references as for the body of the paper---remember that
% references do not count against your page length total.

% \begin{thebibliography}{}
% \setlength{\itemindent}{-\leftmargin}
% \makeatletter\renewcommand{\@biblabel}[1]{}\makeatother
% \bibitem{} J.~Alspector, B.~Gupta, and R.~B.~Allen (1989).
%     \newblock Performance of a stochastic learning microchip.
%     \newblock In D. S. Touretzky (ed.),
%     \textit{Advances in Neural Information Processing Systems 1}, 748--760.
%     San Mateo, Calif.: Morgan Kaufmann.

% \bibitem{} F.~Rosenblatt (1962).
%     \newblock \textit{Principles of Neurodynamics.}
%     \newblock Washington, D.C.: Spartan Books.

% \bibitem{} G.~Tesauro (1989).
%     \newblock Neurogammon wins computer Olympiad.
%     \newblock \textit{Neural Computation} \textbf{1}(3):321--323.
% \end{thebibliography}

\bibliography{references}

\appendix
\onecolumn

\aistatstitle{A Statistical Learning Take on the Concordance Index for Survival Analysis: \\
Supplementary Materials}
\vspace{-17cm}
\textbf{Outline.} The supplementary material is organized as follows. In Section 1 we provide the proofs of the results from Section 3 of the main paper and in Section 2 we provide the proofs of the results from Section 4 of the main paper

\newpage
\section{PROOF OF RESULTS FROM SECTION 3}\label{app:sec1}
Let's first recall the definition of an optimal risk ordering and the assumption $\A, \B, \C$ and $\D$ presented in Section 3 of the paper. 
\begin{definition}[Optimal risk ordering]% \label{def:optimal-risk-ordering} An \textit{optimal risk ordering} is a function $f^\star$ satisfying 
\begin{equation}\label{eq:fullrankingcondition-app}
    f^\star(x) \leq f^\star(x') \Rightarrow \Prob{T>T'|x,x'} \geq \frac{1}{2},
\end{equation}
for all pairs $x, x'$. Note that if condition \eqref{eq:fullrankingcondition-app} is satisfied then it follows directly that $f^\star$ is an optimizer of the C-index and it only depends on the conditional density of events $\mu(t|x)$.
\end{definition}

\begin{assumptionA}
% \label{ass:A}
For all $(x,x')\in\calX^2$, $t\mapsto S(t|x)-S(t|x')$ has constant sign.
\end{assumptionA}

\begin{assumptionB}
% \label{ass:B}
The negative conditional expectation $-\operatorname{CE}(x) = -\Expect{\{T|X=x\}}$ is an optimal risk ordering satisfying~\eqref{eq:fullrankingcondition-app}.
\end{assumptionB}

\begin{assumptionC}
% \label{ass:C}
There exists an optimal ordering $f^\star_\mu$ for survival model $\mu$, satisfying~\eqref{eq:fullrankingcondition-app}.
\end{assumptionC}

\begin{assumptionD}
% \label{ass:D} 
There is not an optimal ordering $f^\star_\mu$ for survival model $\mu$, satisfying~\eqref{eq:fullrankingcondition-app}.
% There exists $m\geq 3$, a subset of indices $\mathcal{I}\subset\{1,\ldots,n\}$, $|\mathcal{I}|=m$, and an ordering $i_1< i_2< \ldots< i_m$ such that, denoting $i_{m+1} = i_1$,
%     \begin{equation*}
%     \min_{k\in\{1,\ldots,m\}}\left(\Prob{T_{i_k}<T_{i_{k+1}}}\right) > \frac{1}{2}.
% \end{equation*}
\end{assumptionD}

\subsection{Proof of Theorem 3.1}
\label{subsec:Proof-th-assumption-A}
We start with the following \cref{lem:stochasticdominance} from \citet{levy2006stochastic}.
\begin{lemma}[\citet{levy2006stochastic}]\label{lem:stochasticdominance} Let $T, T'$ be two random variables with survival functions $S_T$ and $S_{T'}$, respectively. We have that $S_T(t) \geq S_{T'}(t)$ for all $t\geq 0$ if and only if 
$\Expect{}_{T'} \phi(T) \geq \Expect{}_{T'} \phi(T')$ for any non-decreasing function $\phi:\Rspace{}\rightarrow\Rspace{}$.
\end{lemma}

Equipped with~\Cref{lem:stochasticdominance}, we can state the following proposition.
\begin{proposition} \label{prop:assumptionA} Let $T, T'$ be two independent continuous random variables. 
\begin{equation*}
    S_T(t) \geq S_{T'}(t) \hspace{0.2cm}\text{for all}\hspace{0.2cm} t\geq 0 \hspace{0.2cm} \implies \hspace{0.2cm} \Prob{T>T'} \geq \frac{1}{2}.
\end{equation*}
\end{proposition}
\begin{proof}
Since the random variables are continuous, $\Prob{T\geq T'}\geq\frac{1}{2}$ is equivalent to $\Prob{T\geq T'} \geq \Prob{T'\geq T}$. In addition, 
\begin{align*}
    \Prob{T\geq T'} &= \Expect{}1(T\geq T') \\
    &= \Expect{}_{T'}\Expect_{T}\{1(T\geq T')|T'\} \\
    &= \Expect{}_{T'}\Expect_{T}\{1(T\geq T')\} = \Expect{}_{T'}S_T(T'),
\end{align*}
where we used the independence of $T, T'$. Thus, $\Prob{T\geq T'} \geq \Prob{T'\geq T}$ also equivalent to 
\begin{equation}\label{eq:equivalencesurvival}
    \Expect{}_{T'}S_T(T') \geq \Expect{}_{T}S_{T'}(T).
\end{equation}
On the other hand, since the survival curves are uniformly bounded, it holds that 
$\Expect_{T'}S_T(T') \geq \Expect_{T'} S_{T'}(T')$. Applying \Cref{lem:stochasticdominance} to $-S_{T'}$ which is non-decreasing, we obtain that $\Expect_{T'} S_{T'}(T') \geq \Expect{}_{T}S_{T'}(T)$, which concludes the proof.
\end{proof}

Recall the statement of Theorem 3.1 below.
\begin{customthm}{3.1} % \label{th:assumption-A}
If $\mu$ satisfies \cref{ass:A}, the negative conditional expectation
is an optimal risk ordering for the C-index satisfying Condition~\eqref{eq:fullrankingcondition-app}, thus $C_{\mu}(-\operatorname{CE}) = C^{\star}_\mu$.
\end{customthm}

\begin{proof}
Note that, under \Cref{ass:A}, for any $x, x'$, $-\Expect{}\{T|X=x\} \leq -\Expect{}\{T'|X=x'\}$ if and only if $S(t|x) \geq S(t|x')$ for any $t\geq 0$. Using that $\{T|X=x\}$ and $\{T|X=x'\}$ are independent continuous random variables and applying \cref{prop:assumptionA} yields the desired result. 
\end{proof}

\subsection{Proof of Proposition 3.3}
The AFT-H model has the following form:
\begin{equation} \label{eq:aft-h-model-app}
    \log T = f(x) + \sigma(x)\varepsilon,
\end{equation}
where $\varepsilon$ is a standard Gaussian random variable. Recall the statement of Proposition 3.3.
\begin{customprop}{3.3}[AFT-H satisfies $\B$]
% \label{prop:afth-in-B}
Assume $\mu$ is in AFT-H and that if $f(x)<f(x')$ then $\sigma(x)<\sigma(x')$. Then, the negative conditional expectation is an optimal risk ordering, thus
$C_{\mu}(-\operatorname{CE}) = C_\mu^{\star}$.
\end{customprop}
\begin{proof}
We will first show that the negative conditional expectation $-\Expect{}\{\log T~|~X=x\}$ is an optimal risk ordering satisfying \eqref{eq:fullrankingcondition-app}. Following the proof of Corollary 2 by \citet{lebedev2019nontransitivity}, we show that if $T, T'$ are defined as $\log T = a + b\varepsilon, \log T' = a' + b'\varepsilon'$ where $\varepsilon, \varepsilon'$ are symmetric around zero and independent, then $\Prob{T>T'} \geq \frac{1}{2}$ if and only if $a\geq a'$. If we denote by $F_{12}$ the c.d.f of $b\varepsilon - b'\varepsilon'$, which is symmetric around zero too, then 
\begin{equation*}
    \Prob{T>T'} = \Prob{(a + b\varepsilon) - (a' + b'\varepsilon') < 0} = F_{12}(a - a'),
\end{equation*}
where the right hand side is greater or equal than 1/2 if and only if $a\geq a'$. In particular, this means that $-\Expect{}\{\log T~|~X=x\} = -f(x)$ gives an optimal risk ordering. We conclude that $-\Expect{T|X=x}$ is also an optimal risk ordering by noting that if $f(x)<f(x')$, then $\log \Expect{T|X=x} = f(x) + \log \Expect e^{\sigma(x)\varepsilon} \leq f(x') + \log \Expect e^{\sigma(x')\varepsilon} \leq \log \Expect{T'|X=x'}$. This follows from the assumption $f(x)<f(x') \implies \sigma(x)<\sigma(x')$ and that the log moment generating function of a standard Gaussian random variable is monotone and takes the closed form $\log \Expect e^{t\varepsilon} = \frac{1}{2}t^2$.
\end{proof}

\subsection{Proof of Proposition 3.4}
Recall the statement of Proposition 3.4.
\begin{customprop}{3.4}
    % \label{prop:maximizer-closed-form} 
    Assume $\mu$ satisfies \cref{ass:C}. Then, the optimal C-index takes the following form:
\begin{equation*}
    C_{\mu}^\star = C_{\mu}(f^\star_\mu) = \Expect_{X, X'}\phi(\Prob{T>T'|X,X'}), 
\end{equation*}
where $\phi(a) = \max(a, 1-a)$ and $f^\star_\mu$ satisfies \eqref{eq:fullrankingcondition-app}.
\end{customprop}
\begin{proof}
We use the following expression of the C-index: $C(f) = \Expect_{X,X'}\Prob{(f(X)-f(X'))(T-T')<0}$. From this, we obtain that
\begin{align*}
    C(f) &= \Expect_{X,X'}\Prob{T>T'|X,X'}1(r(X,X')< 0) + \Expect_{X, X'}\Prob{T<T'|X,X'}1(r(X,X')>0) \\
    &= \Expect_{X,X'}\Prob{T>T'|X,X'}1(r(X,X')< 0) + \Expect_{X, X'}(1 - \Prob{T>T'|X,X'})1(r(X,X')>0)
\end{align*}
where $r(x, x') = f(x) - f(x')$. Under \cref{ass:C}, if $f^\star$ satisfies \eqref{eq:fullrankingcondition-app}, then the largest term of the right hand side of the above equation is non-zero, leading to the result. 
\end{proof}

\subsection{Proof of Proposition 3.6}
We first recall the definition of the exponential family we presented in the main paper.
\begin{definition}[Exponential family survival model]
% \label{def:scalar-expo-family}
For $\theta:\calX\to\Rspace{}$, $\beta:\Rspace{}_+\to\Rspace{}_+$, $\tau:\Rspace{}_+\to\Rspace{}$, $\eta:\Rspace{}\to \Rspace{}$, and $A:\Rspace{}\to\Rspace{}$ such that, for all $x\in\calX$, the conditional density is a curved exponential family model
\begin{equation} \label{eq:exponential-survival-app}
    \mu(t|x) = \beta(t)\exp\left[\eta\circ\theta(x)\tau(t) - A\circ\theta(x)\right],
\end{equation}
with associated parameter $\theta(x)$. For instance, $\theta(x) = \theta^\top x$ in a generalized linear model.  
\end{definition}
Recall also the statement of Proposition 3.6
\begin{customprop}{3.6}
    % \label{prop:exp-fam}
    Under \Cref{ass:C}, with $\theta$ continuous, $\beta$ positive, $\tau$ non-decreasing and $\eta$ continuously differentiable and non-decreasing, $\theta(x)$ is an optimal risk ordering for the C-index, thus
    $C_\mu(\theta) = C_\mu^{\star}$.
\end{customprop}
\begin{proof}
We need to prove that, for all $x,x'\in\mathcal{X}$, 
    $$\eta\circ\theta(x)\geq \theta(x') \Rightarrow \Prob{T>T'|x,x'} \geq \frac{1}{2}.$$
    We start by showing that, under the assumptions of Proposition~\ref{prop:exp-fam}, the scalar exponential family (conditional) distribution $\mu(t|x) = \rho(t)e^{\eta\circ \theta(x)\tau(t) - A\circ\theta(x)}$ satisfies the Monotone Likelihood Ratio (MLR) property, i.e., for all $x,x'\in\mathcal{X}$ such that $\theta(x)\geq \theta(x')$ and for all $t_1\geq t_0$,
    \begin{equation}
        \label{eq:MLR-app}
        \frac{\mu(t_1|x)}{\mu(t_1|x')}\geq \frac{\mu(t_0|x)}{\mu(t_0|x')}.
    \end{equation}
    Indeed, using the canonical form
    \begin{equation*}
    \mu(t|x) = \beta(t)\exp\left[\eta\circ\theta(x)\tau(t) - A\circ\theta(x)\right],
\end{equation*} 
we have that, for all $x,x'\in\mathcal{X}$ such that $\theta(x)\geq \theta(x')$ and for all $t_1\geq t_0$
    \begin{multline*}
        \frac{\mu(t_1|x)}{\mu(t_1|x')} = \frac{\mathrm{e}^{A\circ\theta(x')}}{\mathrm{e}^{A\circ\theta(x)}}\exp\left\{\tau(t_1)[\eta\circ\theta(x)-\eta\circ\theta(x')]\right\} 
         \geq \frac{\mathrm{e}^{A\circ\theta(x')}}{\mathrm{e}^{A\circ\theta(x)}}\exp\left\{\tau(t_0)[\eta\circ\theta(x)-\eta\circ\theta(x')]\right\} 
        = \frac{\mu(t_0|x)}{\mu(t_0|x')},
    \end{multline*}
    where we have used that $\tau(t_1)\geq \tau(t_0)$ and $\eta\circ\theta(x)\geq\eta\circ\theta(x')$ by assumption to obtain the inequality of the third line.
    
    We now prove that inequality~\eqref{eq:MLR-app} implies $\Prob{T>T'|x,x'}\geq \frac{1}{2}$. Indeed,
    
    \begin{align*}
        \Prob{T>T'|x,x'} = \int_0^{+\infty}\left[\int_{t_0}^{+\infty}\mu(t_1|x)\mu(t_0|x')dt_1\right]dt_0 
        \geq \int_0^{+\infty}\left[\int_{t_0}^{+\infty}\mu(t_1|x')\mu(t_0|x)dt_1\right]dt_0
        = \Prob{T'>T|x,x'} 
    \end{align*}
\end{proof}

\subsection{Proof of Proposition 3.7}
We consider the following Weibull model with varying shape parameter:
\begin{equation}\label{eq:weibull-model-app}
    S(t|x) = e^{-t^{f(x)}}.
\end{equation}
Recall the statement of Proposition 3.7
\begin{customprop}{3.7}% \label{prop:weibull}
    The Weibull model \eqref{eq:weibull-model-app} satisfies \cref{ass:C} but not \cref{ass:B}.
\end{customprop}
\begin{proof}
We have the following identity:
\begin{equation*}
    \Prob{T<T'} = \int_{0}^{+\infty}\exp^{-t^{\alpha_2}}d(1 - \exp^{-x^{\alpha_1}}) = \int_{0}^{+\infty}\exp^{-u^{\alpha_2/\alpha_1} - u}du = I(\alpha_2/\alpha_1),
\end{equation*}
where we have done the change of variables $u = x^{\alpha_1}$. Now we follow the proof of Proposition 1 by \citet{lebedev2019nontransitivity} and show that $I(\beta) \geq \frac{1}{2}$ if and only if $\beta \geq 1$.  The integrand is increasing in $\beta$ for $0<\beta<1$ and decreasing for $\beta>1$. Let $\beta\geq 2$, then
\begin{equation*}
    I(\beta) > \int_{0}^1\exp^{-u^{\beta} - u}du \geq \int_{0}^1\exp^{-u^2 - u}du = \sqrt{\pi}e^{1/4}\left(\Phi\Big(\frac{3\sqrt{2}}{2}\Big) - \Big(\frac{\sqrt{2}}{2}\Big)\right) \approx 0.507 > 1/2,
\end{equation*}
where $\Phi$ is the c.d.f of a standard Gaussian random variable. 
Let $1<\beta<2$. Then the function $u^{\beta - 1}, \beta>0$, is concave and its graph lies below the tangent at $\beta = 1$:
\begin{equation*}
    u^{\beta -1} \leq (\beta - 1)u + (2 - \beta).
\end{equation*}
As a result,
\begin{equation*}
    u^{\beta} \leq (\beta - 1)u^2 + (2 - \beta)u,
\end{equation*}
which can be used for estimating the integral as:
\begin{equation*}
    I(\beta) \geq \int_{0}^{\infty}\exp^{-(\beta - 1)u^2 - (3 - \beta)u}du = 
    \sqrt{\frac{\pi}{\beta - 1}}\exp^{\frac{(3 - \beta)^2}{4(\beta - 1)}}\left(1- \Phi\Big(\frac{3 - \beta}{\sqrt{2(\beta - 1)}}\Big)\right).
\end{equation*}
Plotting this function (see Fig. 3 of \citet{lebedev2019nontransitivity}) we can clearly see how it exceeds 1/2. Moreover, this inequality turns into equality at the limits of the interval \citep{steinhaus1959paradox}. We have thus proved that the shape parameter $f(x)$ is an optimal risk ordering under the Weibull model \eqref{eq:weibull-model-app}, thus \cref{ass:C} is satisfied. On the other hand, \Cref{ass:B} is not satisfied as the expectation of a Weibull is not monotone on the shape parameter. Indeed, the expectation is given by $\Gamma(1 + \frac{1}{f(x)})$, where $\Gamma$ denotes the Gamma function. This function has a minimum between $1.46$ and $1.47$, decreasing first and increasing for larger values. Thus, it gives a different ranking than the optimal risk ordering $f$, and consequently it does not define an optimal risk ordering. \Cref{ass:B} is consequently not satisfied. 
\end{proof}

\subsection{Proof of Proposition 3.8}
Recall the statement of Proposition 3.8.
\begin{customprop}{3.8}% \label{prop:marginal-dependence}
    Under \Cref{ass:D}, the maximizer of the C-index depends on the marginal distribution of the patients covariates $\mu(x)$. 
\end{customprop}

\begin{proof}
As shown in Section 2.1, the maximiser of the C-index can be written as the solution of the following problem:
\begin{equation*}
    \max_{f}~\Expect_{X,X'}{\Prob{T>T'|X,X'}1(f(X)<f(X'))}.
\end{equation*}
Assume for simplicity that the marginal population has support in a finite set of elements $x_1, \ldots, x_n$. Then the optimal ranking of these $n$ points is the solution of the following combinatorial problem over the set of permutations $\calS_n$ of size $n$:
\begin{equation*}
    \max_{\sigma\in\calS_n}~\sum_{i,j=1}^n\gamma_{ij}1(\sigma_i < \sigma_j),
\end{equation*}
where $\gamma_{ij} = p_ip_j\Prob{T>T'|x_i, x_j}$.
This corresponds to the Minimum Weighted Feedback Arc Set (MWFAS) problem with weights $\gamma_{ij}$. Thus, its solution clearly depends on the factor $p_ip_j$ under no further assumption. 
\end{proof}

\newpage
\section{PROOF OF RESULTS FROM SECTION 4}\label{app:sec2}

\subsection{Proof of Theorem 4.1}

We first recall below the definitions of strongly convex functions and of the Fenchel-Young conjugate.
\begin{definition}[Strongly convex]
A twice-differentiable strongly convex function $\Omega$ with parameter $\beta$ defined in a convex domain $\calC\subset\Rspace{}$ is a function satisfying $\Omega{''}(u)\geq \beta$ for all $u$ in the domain $\calC$.
\end{definition}

\begin{definition}[Fenchel-Young conjugate \citep{rockafellar1997convex}] \label{def:fenchel-conjugate}
Let $\Omega$ a function with domain $\calC$. The Fenchel-conjugate $\Omega^{*}$ of $\Omega$ is a convex function defined as
\begin{equation*}
    \Omega^{*}(v) = \sup_{u\in\calC}~vu - \Omega(u). 
\end{equation*}
If $\Omega$ is differentiable, then the domain of $\Omega^*$ corresponds to the image of the gradient $\Omega'$ union the subgradients at the boundary of the domain $\calC$.
If $\Omega$ is strongly convex, then $\Omega^{*}$ is continuously differentiable. 
\end{definition}
Recall now the statement of Theorem 4.1
\begin{customthm}{4.1}% \label{th:cost-functions-B}
    Let $\Omega:\calC\rightarrow\Rspace{}$ be a twice-differentiable \textit{strongly convex} function defined in a closed convex domain $\calC\supseteq \Rspace{}_{+}$ such that $\lim_{|u|\to\infty}\Omega'(u) = +\infty$ where the limit is taken also to the negative real numbers if $\calC$ is unbounded in this direction. Define the cost function $S(v, t) = \Omega^{*}(v) - vt$ where $\Omega^{*}$ is the Fenchel conjugate of $\Omega$ \citep{rockafellar1997convex}. Then, the following risk
\begin{equation}\label{eq:cost-function-conditional-app}
    \calR(f) = \Expect_{(X, U, \Delta)}~\frac{\Delta S(f(X), U)}{G(U)},
\end{equation}
is convex, smooth, and its minimizer is a monotone transformation of the conditional expectation. Thus, under \cref{ass:B} it is Fisher consistent to the C-index. 
\end{customthm}
\begin{proof}
Let's first consider the case without censoring where we observe the pair of random variables $(X, T)$. We want to prove that the minimiser $f^\star$ of the risk $\calR(f) = \Expect_{(X, T)}~S(f(X), T)$ where $S(v, t) = \Omega^{*}(v) - vt$ is a monotone transformation of the conditional expectation. The loss function $S$ has the following properties:
\begin{itemize}
    \item[-] The loss $S$ is a Fenchel-Young loss function \citep{blondel2019learning} with domain $\Rspace{}$ (i.e., $S:\Rspace{}\times\Rspace{}_{+}\rightarrow\Rspace{}$), which follows from the property $\lim_{u\to +\infty}\nabla\Omega(u) = +\infty$ (and also $\lim_{u\to -\infty}\nabla\Omega(u) = +\infty$ if $\calC$ is unbounded to the negative values). indeed for any $v$, there exists $u$ such that $v\in\partial\Omega(u)$ if $u\in\calC$, where $\partial\Omega(u)$ stands for the subgradient of $\Omega$ at point $u$, which equals just the gradient at the interior of the domain $\calC$. Thus the domain of $\Omega^{*}$ is the real line (see \cref{def:fenchel-conjugate}). 
    \item[-] $S$ is smooth as $\Omega$ is strongly convex \citep{rockafellar1997convex} (see also \cref{def:fenchel-conjugate}).
    \item[-] Its minimizer $f^\star$ can be written as $f^\star(x) = \Omega'(\operatorname{CE}(x))$, where $\operatorname{CE}(x) = \Expect{}\{T|X=x\}$ stands for the conditional expectation \citep{blondel2019learning}. Note that the smooth transformation $\Omega':\calC\supseteq\Rspace{}_{+}\rightarrow\Rspace{}$ is monotone due to the convexity of $\Omega$.
\end{itemize} 
When censoring is present, we can use Inverse Censoring Probability Weighting (ICPW) \citep{robins2000IPCW} to construct an unbiased estimate of the risk 
\begin{align*}
    \calR(f) = \Expect_{(X,T)}S(f(X), T) &= \Expect_{(X, T)}\frac{\Expect{}\{ 1(T\leq C)|X\}S(f(X), T)}{\Expect{}\{1(T\leq C)|X\}} \\
    &= \Expect{}_{X}\Expect{}_{T|X}\Expect{}_{C|X}\frac{1(T\leq C)S(f(X), T)}{G(U)} \\
    &= \Expect{}_{(X, U, \Delta)}\frac{\Delta S(f(X), T)}{G(U)},
\end{align*}
where $G(u) = \Prob{C\geq u}$ stands for the censoring survival function. 
\end{proof}

\subsection{Proof of Theorem 4.2}
Recall the statement of Theorem 4.2
\begin{customthm}{4.2}[Excess risk bounds]% \label{th:excess-risk-bound}
    Let $f^\star$ be an optimal risk ranking satisfying \eqref{eq:fullrankingcondition-app}. Let $L> 0$ a positive constant satisfying 
    \begin{equation}\label{eq:bound-L-app}
        |2\Prob{T>T'|x, x'} - 1| \leq L|f^\star(x) - f^\star(x')|, 
    \end{equation}
    for all pairs $x, x'$. Then, the excess risk of the C-index can be bounded as 
    \begin{equation*}
        C(f^\star) - C(f) \leq L\Expect_{X}|f(X) - f^\star(X)|.
    \end{equation*}
\end{customthm}

The result will be based on a reduction to a binary classification problem using the following \cref{lem:binary-classif}.
\begin{lemma}[Theorem 2.2 of \citet{devroye2013probabilistic}]\label{lem:binary-classif}
Let $Y\in\{-1, 1\}$ a binary random variable, and consider the binary classification risk $\calE(r) = \Expect_{Z, Y}~1(r(Z)Y < 0)$. Then, we have that 
\begin{equation*}
    \calE(r) - \calE^\star = \Expect_{Z}|2\Prob{Y|Z}-1|1(r(Z)r^\star(Z)< 0),
\end{equation*}
for any measurable function $r:\calZ\rightarrow\Rspace{}$.
\end{lemma}

    The proof now consists in lifting the C-index ranking problem in $(X, T)$ into a binary classification problem to apply \cref{lem:binary-classif}. Let $Z = (X, X')$ and $Y = \operatorname{sign}(T - T')\in\{-1, 1\}$ and consider the binary classification risk $\calE(r) = \Expect_{Z, Y}~1(r(Z)Y < 0)$ where $r:\calX\times\calX\rightarrow\{-1, 1\}$. Note that the risk $\calE$ is equal to the C-index if we set 
    \begin{equation*}
        r(z) = r((x, x')) = 1(f(x) \geq f(x')).
    \end{equation*}
    This can be easily seen from
    \begin{align*}
        C(f) &= \Expect{}1(f(X) > f(X'), T > T') \\
        &= \frac{1}{2}\Expect{}1((f(X) - f(X'))(T-T') < 0) \\
        &= \frac{1}{2}\Expect{}1(r(Z)Y < 0) = \frac{1}{2}\calE(r).
    \end{align*}
    Moreover, as the maximizer $r^\star$ of the binary classification risk is known to be $r^\star(z) = 1(2\Prob{Y=1|z} \geq 1)$ and $f^\star$ satisfies the optimal risk ranking condition \eqref{eq:fullrankingcondition-app}, we also have $r^\star(z) = 1(f^\star(x) \geq f^\star(x'))$. Thus, applying \cref{lem:binary-classif} we obtain that 
    \begin{equation*}
    C^\star - C(f) = \frac{1}{2}\Expect_{X, X'}\{|2\Prob{T> T'|x, x'} - 1|r_{X, X'}(f, f^\star)\},
    \end{equation*}
    where $r_{x, x'}(f, f^\star) = 1((f^\star(x) - f^\star(x'))(f(x) - f(x'))<0)$. Using now~\eqref{eq:bound-L-app} we obtain 
    \begin{align*}
        &\{|2\Prob{T> T'|x, x'} - 1|r_{x, x'}(f, f^\star)\} \\
        &\leq  L|f^\star(x) - f^\star(x')|r_{x, x'}(f, f^\star)\} \\
        &\leq L |f^\star(x) - f(x) - (f^\star(x') - f(x'))| \\
        & \leq L (|f^\star(x) - f(x)| + |f^\star(x') - f(x')|).
    \end{align*}
    where we have used that $ab\leq 0 \implies |a| \leq |a - b|$ for $a, b\in\Rspace{}$. Finally, taking the expectation over the i.i.d. pairs $X, X'$ on both sides of the inequality we obtain the desired result.

\subsection{Proof of Proposition 4.3}

Recall the statement of Proposition 4.3.
 \begin{customprop}{4.3}% \label{prop:constant-L}
     Condition \eqref{eq:bound-L-app} is satisfied by:
     \begin{itemize}
         \item[(i)] The PH model with $L=1$.
         \item[(ii)] The AFT model with $L$ the Lipschitz constant of the cumulative distribution function of the symmetric random variable $\varepsilon - \varepsilon'$.
\item[(iii)] The AFT-H model \eqref{eq:aft-h-model-app} with $L$ is the Lipschitz constant of the cumulative distribution function of $\varepsilon$ scaled by a factor $a$ where $\sigma(x) \geq 1/a$ for all $x$. 
        \item[(iv)] The exponential family.
    \end{itemize}
\end{customprop}

\paragraph{Cox model (i).}
In the Cox PH model, condition~\eqref{eq:bound-L-app} of \Cref{th:excess-risk-bound} is satisfied with $L= 1$. This comes from the identity 
\begin{align*}
    \Prob{T> T'|x, x'} &= \int_{0}^{+\infty}h(t|x)S(t|x)S(t|x')dt \\
    &= \int_{0}^{+\infty}h_0(t)e^{f(x)}e^{-\int_{0}^th_0(\tau)d\tau (e^{f(x)} + e^{f(x')})}dt =\frac{e^{f(x)}}{e^{f(x)} + e^{f(x')}},
\end{align*}
where at the last identity we have done the change of variables $u = -\int_{0}^th_0(\tau)d\tau$. The final results follows from:
\begin{align*}
    |2\Prob{T> T'|x, x'}-1|= \frac{|e^{f(x)}-e^{f(x')}|}{e^{f(x)}+e^{f(x')}} =\frac{\max(e^{f(x)}, e^{f(x')})}{e^{f(x)}+e^{f(x')}}|f(x) - f(x')| \leq |f(x) - f(x')|.
\end{align*}

\paragraph{AFT model (ii). }
Under the AFT model we have the following identity:
\begin{equation*}
    \Prob{T>T'|x, x'} = \Prob{\log T>\log T'|x, x'} =
    \Prob{f(x) - f(x') > \varepsilon - \varepsilon'} =
    F_{\varepsilon - \varepsilon'}(f(x) - f(x')),
\end{equation*}
where $F_{\varepsilon - \varepsilon'}$ is the cumulative distribution function of the symmetric random variable $\varepsilon - \varepsilon'$. Thus, inequality \eqref{eq:bound-L-app} is satisfied with the Lipschitz constant of the $F_{\varepsilon - \varepsilon'}$.

\paragraph{AFT-H model (iii). } 
Note that the symmetric random variable $\sigma(x)\varepsilon - \sigma(x')\varepsilon'$ follows the same law as $\sqrt{(\sigma(x)^2 + \sigma(x')^2)}\varepsilon$. Using this and performing the same computations as for the AFT model, we obtain the following identity:

\begin{equation*}
    |2\Prob{T>T'|x, x'} - 1| = F_{\varepsilon}\left(\frac{f(x) - f(x')}{\sqrt{\sigma(x)^2 + \sigma(x')^2}}\right).
\end{equation*}
If $\sigma(x) \geq a > 0$ for all $x$, then 
\begin{equation*}
    F_{\varepsilon}\left(\frac{f(x) - f(x')}{\sqrt{\sigma(x)^2 + \sigma(x')^2}}\right) \leq L\left|\frac{f(x) - f(x')}{\sqrt{\sigma(x)^2 + \sigma(x')^2}}\right| \leq \frac{L}{\sqrt{2}a}|f(x) - f(x')|.
\end{equation*}

\paragraph{Exponential family model (iv).} Assume the conditional time-to-event density writes, for $x\in\calX$,
$$\mu(t|x) = \beta(t)\exp\left[\eta\circ\theta(x)\tau(t)-A\circ\theta(x)\right],$$
where $\theta$ and $A$ are continuous functions, and $\eta$ and $\tau$ are continuous, non-decreasing functions. Assume further that $\calX$ is compact (this assumption is satisfied, e.g., if all covariates are bounded). Note that, by continuity of the application $\theta$, $\theta(\calX)$ is also compact; without loss of generality, we assume $\Theta=\theta(\calX)$. Then, the following Proposition holds.
\begin{proposition}
    There exists $L>0$ such that, for all $(x,x')\in\calX^2$,
    $$\left|2\Prob{T>T'|x,x'}-1\right|\leq L\left|\theta(x)-\theta(x')\right|.$$
\end{proposition}
\begin{proof}
Consider the function $G:\Theta\to \Rspace{}_+$ defined by
$$G_{\theta_0}(\theta) = \int_0^{+\infty}S(t;\theta_0)\mu(t;\theta)dt,$$
where, for all $\theta\in\Theta$,
$$\mu(t;\theta) = \beta(t)\exp\left[\eta(\theta)\tau(t)-A(\theta)\right]$$
and
$$S(t;\theta_0) =  \int_{t}^{+\infty}\beta(t_1)\exp\left[\eta(\theta_0)\tau(t)-A(\theta_0)\right]dt_1.$$
First, $t\mapsto S(t;\theta_0)\mu(t;\theta)$ is integrable w.r.t. the variable $t$. Second, $\theta\mapsto S(t;\theta_0)\mu(t;\theta)$ is continuously differentiable w.r.t. the variable $\theta$. Third, since $\Theta$ is compact and $\eta, A$ are continuous, $\eta_-\leq \eta(\theta)\leq \eta_+$ and $A_-\leq A(\theta)\leq A_+$ for all for all $\theta\in\Theta$; we obtain that for all $t\in\Rspace{}_+$ and for all $\theta\in\Theta$, 
\begin{equation*}
    S(t;\theta_0)\mu(t;\theta) \leq S(t;\theta_0)\beta(t)\exp\left[\max\{\eta_-\tau(t), \eta_+\tau(t)\}-A_-\right].
\end{equation*}
% \begin{multline*}
%     S(t;\theta_0)\mu(t;\theta) \leq \\
%     S(t;\theta_0)\beta(t)\exp\left[\max\{\eta_-\tau(t), \eta_+\tau(t)\}-A_-\right].
% \end{multline*}
Thus, by the dominated convergence theorem, $G_{\theta_0}$ is differentiable w.r.t. $\theta$, and 
$$G_{\theta_0}'(\theta) = \int_0^{+\infty}S(t;\theta_0)\mu'(t;\theta)dt.$$

We now use similar arguments to show that $G_{\theta_0}'$ is continuous. First, $t\mapsto S(t;\theta_0)\mu'(t;\theta)$ is integrable w.r.t. the variable $t$. Second, $\theta\mapsto S(t;\theta_0)\mu'(t;\theta)$ is continuous w.r.t. the variable $\theta$. Third, since $\Theta$ is compact and $\eta, A$ are continuous, $\eta_-\leq \eta(\theta)\leq \eta_+$ and $A_-\leq A(\theta)\leq A_+$ for all for all $\theta\in\Theta$; we obtain that for all $t\in\Rspace{}_+$ and for all $\theta\in\Theta$,

\begin{align}
    \label{eq:unif-cvg-1}
    S(t;\theta_0)\mu(t;\theta) &\leq
    S(t;\theta_0)\beta(t)\left(\max\{\eta_-\tau(t), \eta_+\tau(t)\}-A_-\right)\\
    &\times\exp\left(\max\{\eta_-\tau(t), \eta_+\tau(t)\}-A_-\right).
\end{align}
By the dominated convergence theorem, $G_{\theta_0}'$ is continuous w.r.t. $\theta$. We obtain that $G_{\theta_0}$ is continuously differentiable w.r.t. $\theta$ on a bounded domain. It is thus Lipschitz continuous, and there exists a positive constant $L_{\theta_0}$ such that, for all $\theta, \tilde\theta \in \Theta$,
$$|G_{\theta_0}(\theta) - G_{\theta_0}(\tilde\theta)|\leq L_{\theta_0}|\theta-\tilde\theta|,$$
where $|\cdot|$ denotes the Euclidean norm on $\Theta$. 

The final step consists in proving that the family of constants $\{L_{\theta_0}\}_{\theta_0\in\Theta}$ are uniformly bounded. To do so, it suffices to show that the derivative $G_{\theta_0}'(\theta)$ is bounded uniformly for all $\theta_0$ and for all $\theta$. Using~\eqref{eq:unif-cvg-1}, we obtain that 

\begin{align*}%\label{eq:unif-cvg-2}
    G_{\theta_0}'(\theta) & \leq 
    \int_{0}^{+\infty}\{S(t;\theta_0)\beta(t)\left(\max\{\eta_-\tau(t), \eta_+\tau(t)\}-A_-\right)\\
    & \times\exp\left(\max\{\eta_-\tau(t), \eta_+\tau(t)\}-A_{-}-\right)\}dt.
\end{align*}
Noticing that 
\begin{equation*}
    S(t;\theta_0)\leq 
    \int_{t}^{+\infty}\beta(t)\exp\left[\max\{\eta_-\tau(t_1), \eta_+\tau(t_1)\}-A_-\right]dt_1.
\end{equation*}
completes the proof.
\begin{comment}
On the one hand, $\mu(t;\theta')$ is integrable by definition; it is also $C^\infty$ w.r.t. $(t,\theta')$. On the other hand, $S(t;\theta)$ is also integrable $C^\infty$ w.r.t. $(t,\theta)$ (see also, e.g., Theorem 5.8 in~\citet{lehmann2005theory}). 

    We conclude that $G(\theta,\theta')$ is $C^\infty$ w.r.t. $(\theta,\theta')$. In addition, $\theta(\calX)$ is compact as the image of a compact set by a continuous function; as a result, $G(\theta,\theta')$ has bounded derivatives on $\theta(\calX)^2$ and is thus Lipschitz continuous on $\theta(\calX)^2$—equipped with the Euclidean norm $|(\theta,\theta')-(\theta_1,\theta_1')|= \sqrt{|\theta-\theta_1|^2 + |\theta'-\theta_1'|^2}$—with Lipschitz constant denoted by $L$. Thus, for all $(\theta,\theta'),(\theta_1,\theta_1')$ in $\theta(\calX)^2$, 
$$|G(\theta,\theta')-G(\theta_1,\theta_1')|\leq L \sqrt{|\theta-\theta_1|^2 + |\theta'-\theta_1'|^2}.$$
Applying the inequality to $(\theta,\theta)$ and $(\theta,\theta')$, and noticing that $G(\theta,\theta) = \frac{1}{2}$ we obtain that, for all $\theta,\theta'\in\theta(\calX)$, 
$$|2G(\theta,\theta')-1|\leq 2L|\theta-\theta'|.$$
In particular, for any $(x,x')\in\calX^2$, by continuity of the application $\theta:\calX\to \theta(\calX)$, we can apply the inequality to $\theta = \theta(x)$ and $\theta' = \theta(x')$. We obtain that
$$|2\Prob{T>T'|x_0,x}-1|\leq 2L|\theta(x_0)-\theta(x)|,$$
which completes the proof
\end{comment}
\end{proof}

\subsection{Proof of Theorem 4.5}
Recall the statement of Theorem 4.5
\begin{customthm}{4.5}% \label{th:excess-risk-conditional-expectation}
   Let $\calR$ be the risk defined in \eqref{eq:cost-function-conditional-app}.  Under \cref{ass:B}, the following inequality holds:
    \begin{equation*}
        C^\star - C(f) \leq 2L\gamma\sqrt{\calR(f) - \calR^\star},
    \end{equation*}
    where $\Omega''(u)\geq 1/\gamma^{2}$ for all $u$ in the domain $\calC$.
\end{customthm}
\begin{proof}
We can safely assume there is no censoring ($\calR(f) = \Expect{}_{(X, T)}S(f(X), T)$); if there is censoring we can just use the ICPW technique described in the proof of Theorem 4.1 to write the expectation in terms of $(X, U, \Delta)$. We assume that \eqref{eq:bound-L-app} is satisfied by the conditional expectation $\operatorname{CE}$ as
$|2\Prob{T>T'|x, x'} - 1| \leq L|\operatorname{CE}(x) - \operatorname{CE}(x')|$,
so applying \cref{th:excess-risk-bound} we obtain 
 \begin{equation*}
    C(f^\star) - C(h) \leq L\Expect_{X}|h(X) - \operatorname{CE}(X)|.
\end{equation*}

The result is based on the excess risk bound proof by \cite{nowak2019general}. If $S$ is a Fenchel-Young loss \cite{blondel2019learning} then we can write 
\begin{equation*}
    \calR(f) - \calR^\star = \Expect_{X}~D_{\Omega}(\operatorname{CE}(X), \nabla\Omega^*(f(X))),
\end{equation*}
where $D_{\Omega}(u, v) = \Omega(u) - \Omega(v) - \nabla \Omega(v)'(u - v)$ is the \textit{Bregman divergence} of the convex function $\Omega$ at the pair of points $u, v$ \citep{blondel2019learning}. Moreover, if $\Omega$ is $1/\gamma^2$ strongly convex (i.e., $\Omega''\geq 1/\gamma^2$), then 
\begin{equation*}
    (u - v)^2 \leq 2\gamma^2 D_{\Omega}(u, v).
\end{equation*}
Then, the final result follows from Cauchy-Schwartz:
 \begin{align*}
    C(f^\star) - C(f) &= C(f^\star) - C(\nabla \Omega^*(f)) \\
    &= \leq L\Expect_{X}|\nabla \Omega^*(f(X)) - \operatorname{CE}(X)| \\
    &\leq L\sqrt{\Expect_{X}|\nabla \Omega^*(f(X)) - \operatorname{CE}(X)|^2} \\
    &\leq \sqrt{2}L\gamma\sqrt{\Expect_{X}D_{\Omega}(\operatorname{CE}(X), \nabla\Omega^*(f(X)))} = \sqrt{2}L\gamma\sqrt{\calR(f) - \calR^\star}.
\end{align*}
\end{proof}

% \bibliography{references}
\end{document}